\documentclass{article}
\usepackage[utf8]{inputenc}
\pdfoutput=1
\usepackage[sc,osf]{mathpazo}
\usepackage[height=8.5in,width=6in,letterpaper]{geometry}
\usepackage[sort&compress,numbers]{natbib}
\usepackage[colorlinks=true,citecolor=blue,breaklinks]{hyperref}
\usepackage[hyphenbreaks]{breakurl} 
\usepackage[tracking]{microtype}

\makeatletter
\newlength\aftertitskip     \newlength\beforetitskip
\newlength\interauthorskip  \newlength\aftermaketitskip

\def\maketitle{\par
 \begingroup
   \def\thefootnote{\fnsymbol{footnote}}
   \def\@makefnmark{\hbox to 4pt{$^{\@thefnmark}$\hss}}
   \@maketitle \@thanks
 \endgroup
\setcounter{footnote}{0}
 \let\maketitle\relax \let\@maketitle\relax
 \gdef\@thanks{}\gdef\@author{}\gdef\@title{}\let\thanks\relax}

\def\@startauthor{\noindent \normalsize\bf}
\def\@endauthor{}
\def\@starteditor{\noindent \small {\bf Editor:~}}
\def\@endeditor{\normalsize}
\def\@maketitle{\vbox{\hsize\textwidth
 \linewidth\hsize \vskip \beforetitskip
 {\begin{center} \LARGE\@title \par \end{center}} \vskip \aftertitskip
 {\def\and{\unskip\enspace{\rm and}\enspace}%
  \def\addr{\small\it}%
  \def\email{\hfill\small\tt}%
  \def\name{\normalsize\bf}%
  \def\AND{\@endauthor\rm\hss \vskip \interauthorskip \@startauthor}
  \@startauthor \@author \@endauthor}
}}

\makeatother

\pdfoutput=1                    

\usepackage{graphicx}
\usepackage{amsmath,amssymb,amsfonts,amsxtra,mathrsfs,bm}
\usepackage{amsthm}
\pdfoutput=1
\usepackage[table]{xcolor}
\usepackage{algorithm}
\usepackage{algpseudocode}
\usepackage{appendix}
\usepackage{subfigure}
\usepackage{multicol}
\usepackage{multirow}
\usepackage[neverdecrease]{paralist}
\usepackage{xspace}
\usepackage{caption}
\usepackage{graphicx}
\usepackage{enumerate}
\usepackage{units}
\usepackage{hyperref}       
\usepackage{nicefrac}

\definecolor{darkblue}{rgb}{0.0,0.0,0.85}
\hypersetup{
  colorlinks = true,
  citecolor  = darkblue,
  linkcolor  = darkblue,
  citecolor  = darkblue,
  filecolor  = darkblue,
  urlcolor   = darkblue,
}

\usepackage{pgfplots}

\newcommand{\reals}{\mathbb{R}}
\newcommand{\E}{\mathbb{E}}

\newcommand{\norm}[1]{\|{#1}\|}
\newcommand{\nlsum}{\sum\nolimits}

\newcommand{\ip}[2]{\langle {#1},\, {#2} \rangle}

\newcommand{\half}{\tfrac{1}{2}}
\newcommand{\nhalf}{\nicefrac{1}{2}}

\newcommand{\set}[1]{\{ #1\}}
\newcommand{\pp}{\mathbb{P}}

\newcommand{\Mc}{\mathcal{M}}
\newcommand{\Nc}{\mathcal{N}}
\newcommand{\Ac}{\mathcal{A}}
\newcommand{\Lc}{\mathcal{L}}
\newcommand{\Oc}{\mathcal{O}}
\newcommand{\mr}{\bm{R}}
\newcommand{\ms}{\bm{S}}
\newcommand{\mU}{\bm{U}}
\newcommand{\vx}{\bm{x}}

\newcommand{\vy}{\bm{y}}
\newcommand{\vt}{\bm{t}}

\newcommand{\vmu}{\bm{\mu}}
\newcommand{\valpha}{\bm{\alpha}} 
\newcommand{\vlambda}{\bm{\lambda}} 
\newcommand{\msigma}{\bm{\Sigma}} 
\newcommand{\mlambda}{\bm{\Lambda}} 
\newcommand{\mpsi}{\bm{\Psi}} 

\newcommand{\gmm}{\textsc{Gmm}\xspace}
\newcommand{\EM}{\textsc{EM}\xspace}

\DeclareMathOperator{\Exp}{Exp}
\DeclareMathOperator{\retr}{Ret}

\DeclareMathOperator{\trace}{tr}

\theoremstyle{plain}
\newtheorem{theorem}{Theorem}
\newtheorem{corollary}[theorem]{Corollary}
\newtheorem{proposition}[theorem]{Proposition}
\newtheorem{lemma}[theorem]{Lemma}

\theoremstyle{definition}

\theoremstyle{remark}

\numberwithin{equation}{section}

\graphicspath{{./fig/}}

\def\tehr{School of ECE, College of Engineering, University of Tehran, Tehran, Iran.}
\def\mitu{Massachusetts Institute of Technology, Cambridge, MA, USA.}

\begin{document}

\title{An Alternative to EM for Gaussian Mixture Models: Batch and Stochastic Riemannian Optimization}

\author{\name Reshad Hosseini \email{reshad.hosseini@ut.ac.ir}\\
  \addr{\tehr}\\
  \name Suvrit Sra\thanks{S. Sra acknowledges partial support from NSF-IIS-1409802} \email{suvrit@mit.edu}\\
  \addr{\mitu}
}

%


\maketitle

\begin{abstract}
  We consider maximum likelihood estimation for Gaussian Mixture Models (\gmm{}s). This task is almost invariably solved (in theory and practice) via the Expectation Maximization (EM) algorithm. EM owes its success to various factors, of which is its ability to fulfill positive definiteness constraints in closed form is of key importance. We  propose an alternative to EM by appealing to the rich Riemannian geometry of positive definite matrices, using which we cast \gmm{} parameter estimation as a Riemannian optimization problem. Surprisingly, such an out-of-the-box Riemannian formulation completely fails and proves much inferior to EM. This motivates us to take a closer look at the problem geometry, and derive a better formulation that is much more amenable to Riemannian optimization. We then develop (Riemannian) batch and stochastic gradient algorithms that outperform EM, often substantially. We provide a non-asymptotic convergence analysis for our stochastic method, which is also the first (to our knowledge) such global analysis for Riemannian stochastic gradient. Numerous empirical results are included to demonstrate the effectiveness of our methods.
\end{abstract}

\section{Introduction}
Gaussian Mixture Models are extensively used across many tasks in machine learning, signal processing, and other areas~\citep{dudahart,keener,bishop,murphy12,McLPee00,reynolds2000speaker,friedman2001elements}. For a vector $\vx \in \reals^d$, the density of a Gaussian Mixture Model (\gmm{}) is given by
\begin{equation}
  \label{eq:1}
  p(\vx) := \nlsum_{j=1}^K\alpha_j p_{\Nc}(\vx; \vmu_j, \msigma_j),
\end{equation}
where $p_{\Nc}$ is a Gaussian with mean $\vmu \in \reals^d$ and covariance $\msigma \succ 0$, i.e.,
\begin{equation*}
  p_{\mathcal{N}}(\vx;\vmu,\msigma) := \det(\msigma)^{-1/2}(2\pi)^{-d/2} \exp \bigl( -\tfrac12(\vx-\vmu)^T \msigma^{-1} (\vx-\vmu) \bigr).
\end{equation*}
Given i.i.d.\ samples $\set{\vx_1,\ldots,\vx_n}$ drawn from~\eqref{eq:1}, we seek maximum likelihood estimates $\set{\hat{\vmu}_j \in \reals^d, \hat{\msigma}_j \succ 0}_{j=1}^K$ and $\hat{\valpha} \in \Delta_K$ of the parameters of the \gmm. This estimation is cast as the following log-likelihood maximization problem:
\begin{equation}
  \label{eq:2}
  \max_{\valpha \in \Delta_K,\set{\vmu_j,\msigma_j \succ 0}_{j=1}^K}\quad
  \sum_{i=1}^n\log\Bigl(\nlsum_{j=1}^K\alpha_j p_{\Nc}(\vx_i; \vmu_j,\msigma_j)\Bigr).
\end{equation}

A quick literature search reveals that~\eqref{eq:2} is most frequently solved via the Expectation Maximization (\EM) algorithm~\citep{dempster77} or its variants. Although other optimization methods have also been considered~\citep{redWal84},  for solving practical instances of~\eqref{eq:2} usual methods such as conjugate gradients, quasi-Newton, Newton, are typically regarded as inferior to \EM~\citep{jordan96}.

\paragraph{Difficulties and Motivation.} The primary reason why standard nonlinear methods have difficulties in solving~\eqref{eq:2} is the positive definiteness constraint on the covariance matrices. Since this constraint defines an open subset of Euclidean space, in principle, if the iterates remain in the interior, standard unconstrained Euclidean optimization methods could be used. The iterates may, however, approach the boundary of the constraint set, especially in higher dimensions, which can lead to very slow convergence. One approach is to formulate the positive definite constraint via a set of smooth convex inequalities~\citep{vanderbei2000formulating} and use interior-point methods. It was observed in~\citep{sra2013geometric} that using such sophisticated methods can be vastly slower (on some closely related statistical problems) than simpler \EM-like fixed-point methods, especially with growing problem dimensionality.

Another ``natural'' approach to handle the positive definite constraint is to use the Cholesky decomposition, as was exploited for semidefinite programming in~\citep{burer1999solving}, and more recently in~\citep{bhoj16}. In general, this decomposition can add spurious local maxima and stationary points to the objective function of general optimization problems, even for semidefinite programs~\citep{vanderbei2000formulating}. Remarkably, it can be shown that such a decomposition does not add spurious local maxima to~\eqref{eq:2}. Nevertheless, we observed (empirically) that the convergence speed of standard nonlinear solvers for estimating parameters of~\eqref{eq:2} using Cholesky decomposition is considerably slower than \EM.

Motivated by the success of non-Euclidean optimization for some problems with positive definite variables~\citep{sra2013geometric,sra15}, we consider an alternative approach to \EM. In particular, we solve~\eqref{eq:2} via \emph{Riemannian optimization}. Surprisingly, a na\"ive use of Riemannian methods completely fails to compete with \EM, while their use on a careful  reformulation\footnote{A preliminary  version of this work appeared at the \emph{Advances in Neural Information Processing Systems (NIPS 2015)}, wherein this reformulation was originally introduced.} of~\eqref{eq:2} demonstrably succeeds.

We describe this reformulation in Section~\ref{sec:prob}, and remark here informally on why a na\"ive use of manifold optimization fails: The negative log-likelihood for a single Gaussian is Euclidean convex (the key property that makes the ``M-step'' of \EM easy), but not geodesically convex. Reformulating the problem to remove this geometric mismatch might therefore be fruitful, i.e., if we reformulate the single Gaussian likelihood to be geodesically convex, manifold optimization may benefit. This intuition turns out to have remarkable empirical consequences as will become apparent from the paper. 

\vskip8pt
\noindent\textbf{Contributions.} The present paper goes substantially beyond our preliminary work~\citep{hosseini2015matrix} in several important aspects. Let us therefore outline our main contributions below. 
\begin{list}{\footnotesize$\blacktriangleright$}{\leftmargin=2em}
\setlength{\itemsep}{-1pt}
\item We develop reformulations not only for \gmm{}s, but also for richer likelihood models that incorporate conjugate priors. 
\item We present both batch and stochastic optimization algorithms; the latter greatly enhances the scalability of our methods. Moreover, our methods permit the use of retractions (beyond the usual exponential map) and vector transport, which enables further scalability.
\item We provide an iteration complexity analysis of stochastic gradient on manifolds, obtaining a $O(1/\sqrt{T})$ bound. To our knowledge, this is the first \emph{non-asymptotic} convergence analysis for stochastic gradient on manifolds. Subsequently, we present analysis that outlines why Riemannian SGD applies to penalized \gmm-likelihood maximization.
\end{list}

We provide experimental evidence on several real-data comparing manifold optimization to \EM. As may be gleaned from our results, manifold optimization performs well across a wide range of parameter values and problem sizes, while being much less sensitive to overlapping data than \EM, and while displaying less variability in running times. 

We review key concepts of first-order deterministic manifold optimization. We also include the design and specific implementation choices of our line-search procedure. These choices ensure convergence, and are instrumental to making our Riemannian-LBFGS solver outperform both \EM and Riemannian conjugate gradients. This solver should be of independent interest too.

We will also release a \textsc{Matlab} implementation of the methods developed in this paper. The manifold CG method that we use is directly based on the excellent toolkit \textsc{ManOpt}~\citep{boumal2014manopt}.

\subsection{Related work} \EM is such a widely studied method, that we have no hope of summarizing all the related work, even if we restrict to just \gmm{}s. Let us instead mention a few lines of related work. \citet{jordan96} examine several aspects of \EM for \gmm{}s and counter the claims of~\citet{redWal84}, who thought \EM to be inferior to general purpose nonlinear programming methods, especially second-order methods. However, it is well-known (see e.g.,~\citep{jordan96,redWal84}) that \EM can attain good likelihood values rapidly, and that it scales to larger problems than amenable to second-order methods. Local convergence analysis of \EM is available in~\citep{jordan96}, with more refined and precise results in~\citep{ma2000asymptotic}, who formally show that when data have low overlap, \EM can converge locally superlinearly. Our paper uses manifold LBFGS, which being a quasi-Newton method can also display local superlinear convergence, though this capability is not the focus of our paper.

Parameter fitting using gradient-based methods has also been suggested~\citep{naim2012convergence,salakhutdinov2003optimization}. Here, to satisfy positive definiteness, the authors suggest using Cholesky decompositions. These works report results only for low-dimensional problems and spherical (near spherical) covariance matrices. 

Beyond \EM, there is also substantial work on theoretical analysis of \gmm{}s \citep{dasgupta1999learning,moitra2010,kakade15,balakrishnan2014statistical}. These studies are theoretically valuable (though sometimes limited to either low-dimensional, or small number of mixture components, or spherical Gaussians, etc.), but orthogonal to our work which focuses on practical numerical algorithms for general \gmm{}s.

The use of Riemannian optimization for \gmm is relatively new, even though manifold optimization is by now a fairly well-developed branch of optimization. A classic reference is~\citep{udriste}; a more recent work is~\citep{absil2009optimization}; and even a \textsc{Matlab} toolbox exists now~\citep{boumal2014manopt}. In machine learning, manifold optimization has witnessed increasing interest\footnote{Not to be confused with ``manifold learning'' a separate problem altogether.}, e.g., for low-rank optimization~\citep{vandereycken2013low,journee2010low}, optimization based on geodesic convexity~\citep{sra2013geometric,wiesel12}, or for neural network training~\citep{wisdom2016full}.

\section{Background on manifold optimization}
Manifolds are spaces that locally resemble a Euclidean space, and smooth manifolds have smooth transitions between locally Euclidean-like subsets~\citep{lee12}. The tangent space $T_x$ is an approximating vector space at each point $x$ of the manifold $\Mc$. The tangent bundle of a smooth manifold $\Mc$ is a manifold $T \Mc$, which assembles all the tangents in that manifold, $T \Mc = \bigsqcup_{x \in \Mc} T_x = \{ (x,y)|x\in \Mc, y \in T_x\}$. If a smooth manifold is equipped with a smoothly-varying inner product on each of its tangent spaces, it is called \emph{Riemannian manifold}. 

This additional structure of a Riemannian manifold proves very useful in developing optimization techniques specific to  manifolds~\citep{udriste}. Indeed, it is easy to extend unconstrained optimization techniques to smooth manifolds, at least from the perspective of asymptotic complexity analysis~\citep{absil2009optimization}; though the non-asymptotic case is considerably more complicated~\citep{zhangSra16a,rsvrg}. 

The key manifold in this paper is $\pp^d$, the manifold of $d\times d$ symmetric positive definite (PSD) matrices. At a point $\msigma \in \pp^d$, the tangent space $T_{\msigma}$ is isomorphic to the entire set of symmetric matrices; and the Riemannian metric at $\msigma$ between two vectors $\xi$ and $\eta$ in $T_{\msigma}$ is given by $g_{\msigma}(\xi,\eta) := \trace(\msigma^{-1}\xi\msigma^{-1}\eta)$. 

Riemannian manifolds have \emph{geodesics}, which are curves that (locally) join points along shortest paths which depends on the choice of Riemannian metric. Geodesics help generalize the notion of convexity to manifolds. 

\subsection{Geodesic convexity}
\label{sec:manifold}
Let $\Mc$ be a Riemannian manifold and $\gamma_{xy}$ a geodesic from $x$ to $y$; that is
\begin{equation*}
  \gamma_{xy}:[0,1] \to \Mc,\quad \gamma_{xy}(0) = x,\ \gamma_{xy}(1) = y.
\end{equation*}
A set $\Ac \subseteq \Mc$ is \emph{geodesically convex} (henceforth g-convex) if for all $x, y \in \Ac$ there is a geodesic $\gamma_{xy}$ contained within $\Ac$. Further, a function $f: \Ac \to \reals$ is g-convex if for all $x, y \in \Ac$, the composition $f \circ \gamma_{xy}: [0,1] \to \reals$ is convex in the usual Euclidean sense.

The Riemannian metric on $\pp^d$ mentioned above induces a geodesic between two points $\msigma_1$ and $\msigma_2$ that has the well-known closed-form (see e.g.,~\citep[Ch.~6]{bhatia07}):
\begin{equation*}
  \gamma_{\msigma_1,\msigma_2}(t) := \msigma_1^{\nhalf}\left(\msigma_1^{-\nhalf}\msigma_2\msigma_1^{-\nhalf}\right)^t\msigma_1^{\nhalf},\quad 0 \le t \le 1.
\end{equation*}
Thus, a function $f: \pp^d \to \reals$ if g-convex on $\pp^d$ if it satisfies
\begin{equation*}
  f(\gamma_{\msigma_1,\msigma_2}(t)) \le (1-t)f(\msigma_1) + tf(\msigma_2),\qquad t \in [0,1],\ \msigma_1,\msigma_2 \in \pp^d.
\end{equation*}
The negative of a g-convex function is called  g-concave. For a g-convex function, local optimality implies global optimality even if it is nonconvex in the Euclidean case. This remarkable property follows easily from g-convexity upon mimicking the corresponding Euclidean proof. This property has been investigated in some matrix theoretic applications~\citep{bhatia07,sra15}, and has been used in recent theoretical and applied works in nonlinear optimization~\citep{ring2012optimization,sra2013geometric,wiesel12,zhangSra16a}. 

\subsection{First-order methods for Riemannian optimization}
\begin{figure}
  \centering
  \includegraphics[scale=0.5]{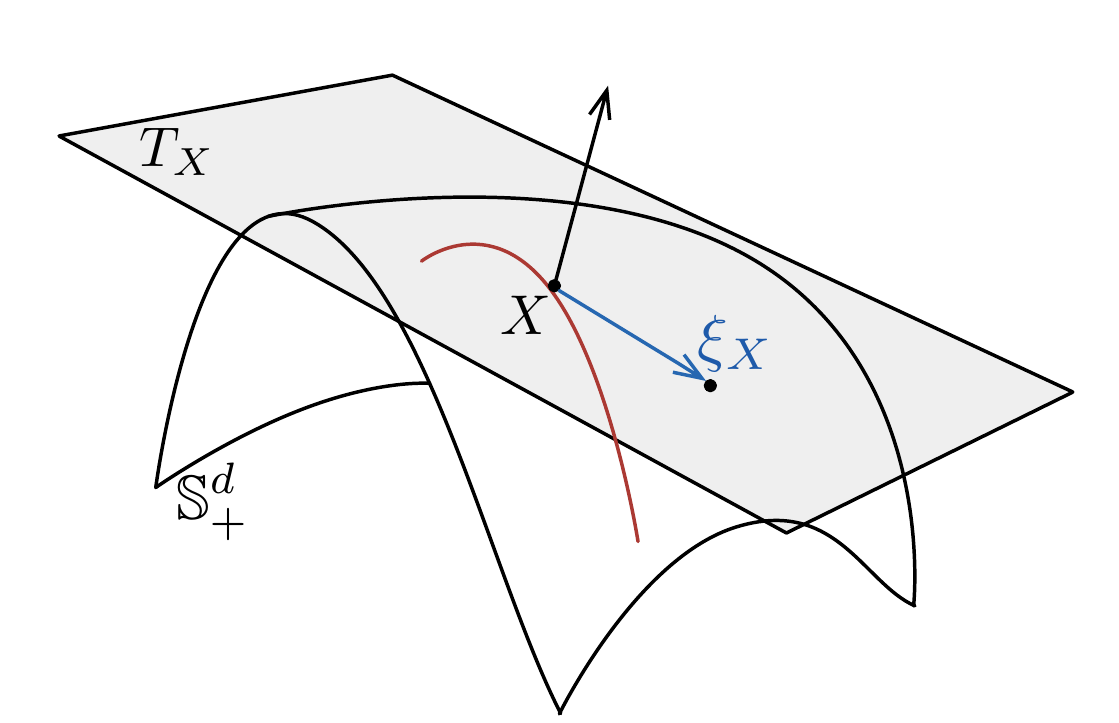}
  \caption{\footnotesize Visualization of line-search on a manifold: $X$ is a point on the manifold, $T_X$ is the tangent space at the point $X$, $\xi_X$ is a descent direction at $X$; the red curve is the curve along which line-search is performed. }
  \label{fig:cg}
\end{figure}
\label{sec:manopt}
At a high-level, first-order methods for manifold optimization methods operate iteratively as follows (see Fig.~\ref{fig:cg} for a conceptual demonstration):
\begin{itemize}
\item[i)] Obtain a descent direction, namely, a vector in tangent space that decreases the cost function if we infinitesimally move along it; 
\item[ii)] Perform a line-search along a smooth curve on the manifold to obtain sufficient decrease and ensure convergence.
\end{itemize}

Such a smooth curve that is parametrized by a point on the manifold and a (descent) direction is called retraction.
A retraction is a smooth mapping $\retr$ from the tangent bundle $T\Mc$ to the manifold $\Mc$. The restriction of retraction to $T_x$, $\retr_x:T_x \rightarrow \Mc$, is a smooth mapping with 
\begin{itemize}
\item[1)]$\retr_x(0)=x$, where $0$ denotes the zero element of $T_x$. 
\item[2)]$D\retr_x(0)=\text{id}_{T_x}$, where $D\retr_x$ denotes the derivative of $\retr_x$ and $\text{id}_{T_x}$ denotes the identity mapping on $T_x$. 
\end{itemize}
One possible candidate for retraction on Riemannian manifolds is the exponential map. The exponential map $\Exp_x:T_x\rightarrow\Mc$ is defined as $\Exp_x v = \gamma(1)$, where $\gamma$ is the geodesic satisfying the conditions $\gamma(0)=x$ and $\dot{\gamma}(0) = v$. The reader is referred to~\citep{absil2009optimization,udriste} for more in depth discussion.

First-order methods are based on gradients. The gradient on a Riemannian manifold is defined as the vector $\nabla f(x)$ in tangent space such that
\begin{equation*}
Df(x)\xi = \ip{\nabla f(x)}{\xi},\quad\text{for}\ \xi \in T_x,
\end{equation*} 
where $\ip{\cdot}{\cdot}$ is the inner product in the tangent space $T_x$.

Another important concept needed for methods like conjugate-gradient and LBFGS is \emph{vector transport}. Vector transport is a smooth function that allows moving tangent vectors along retractions. 
A vector transport $\mathcal{T}:\Mc \times \Mc \times T \Mc\rightarrow T \Mc, (x,y,\xi) \mapsto \mathcal{T}_{x,y}(\xi)$
is a  mapping satisfying the following properties: 
\begin{itemize}
\item[1)] There exists an associated retraction $\retr$ and a  tangent vector $\nu$ satisfying $\mathcal{T}_{x,y}(\xi) \in T_{\retr_x(\nu)}$, for all $\xi \in T_x$. 
 \item[2)] $\mathcal{T}_{x,y}\xi = \xi$, for all $\xi \in T_x$. 
 \item[3)]  The mapping $\mathcal{T}_{x,y}(.)$ is linear.
 \end{itemize}

An important special case of vector transport is parallel transport, which is defined as a differential map between tangent spaces at different points on the manifold with zero derivative along a smooth curve connecting the points. The differential map between tangent spaces on the manifold is a smooth vector field, where a vector field is an assignment of a tangent vector to each point on a manifold. For computing the derivative of such a map, one first needs to define a connection, which is a way to perform directional derivative of vector fields. Let $\mathcal{V}(\Mc)$ be the set of smooth vector fields on $\Mc$, a connection is a map $\nabla:\mathcal{V}(\Mc) \times \mathcal{V}(\Mc) \rightarrow \mathcal{V}(\Mc)$ satisfying certain properties~\citep{absil2009optimization}. Given a smooth curve $\gamma:[0,1]\rightarrow \Mc$, transporting a vector $\nu_0 \in T_{\gamma(0)}$ to a vector $\nu(t) \in T_{\gamma(t)}$ can be done by solving the following initial value problem
\begin{equation*}
\nabla_{\dot{\gamma}(t)}\nu =0,\quad \nu(0)=\nu_0.
\end{equation*}
For $x=\gamma(0)$ and $y=\gamma(t)$, the parallel transport of $\nu_0 \in T_x$ to $\nu(t) \in T_y$ is a vector transport $\nu(t) = T_{x,y}\nu_0$.

Table~\ref{tbl:psdSummary} summarizes the key quantities for $\pp^d$. If the parameter space is a product space of several manifolds, the concepts can be easily defined based on individual manifolds. For example, the exponential map, gradient and parallel transport are defined as the Cartesian product of individual expressions, and the inner product is defined as the sum of inner product of the components in their respective manifolds.
\begin{table}[t]
  \caption{Summary of Riemannian expressions for PSD matrices}
  \label{tbl:psdSummary}
  \begin{tabular}{l|l}
      \hline
    Definition & Expression for the PSD manifold\\
    \hline
    Tangent space& Space of symmetric  matrices \\
    Metric between $\xi,\eta$ at $\Sigma$ & $g_{\Sigma}(\xi,\eta)= \trace (\Sigma^{-1} \xi \Sigma^{-1} \eta)$  \\
    Gradient at $\Sigma$ if Euclidean gradient is $\nabla_Ef$ & $\nabla f(\Sigma) = \tfrac12 \Sigma (\nabla_Ef(\Sigma)+[\nabla_Ef(\Sigma)]^T) \Sigma$\\
    Exponential map at $\Sigma$ in direction $\xi$ & $\Exp_{\Sigma}(\xi)= \Sigma \exp (\Sigma^{-1} \xi )$  \\
    Parallel transport of  $\xi$ from $\Sigma_1$ to $
    \Sigma_2 $& $\mathcal{T}_{\Sigma_1,\Sigma_2}(\xi)= E \xi E^T,\quad E =(\Sigma_2\Sigma_1^{-1})^{1/2}$\\
  \end{tabular}
\end{table}

Two typical line-search methods are used in practice, one is Armijo rule and the other is line-search algorithm satisfying Wolfe conditions. For the case of LBFGS method, it is more common to use Wolfe line-search because it can guarantee that each step of LBFGS creates a descent direction~\citep{ring2012optimization}.

\subsection{Wolfe line-search}
The first Wolfe condition is a sufficient-decrease condition and is given by
\begin{equation*}
f(\retr_{x_k}(\alpha \xi_k)) \leq f(x_k) + c_1 \alpha D f(x_k) \xi_k, 
\end{equation*}
where $0<c_1<1$ is a constant typically chosen to be around $10^{-4}$ for LBFGS. This condition alone does not ensure that the algorithm makes sufficient progress. Another condition called curvature condition is needed,
\begin{equation}
D f(\retr_{x_k}(\alpha \xi_k)) \mathcal{T}_{x_k,\retr_{x_k}(\alpha \xi_k)}(\xi_k) \geq c_2 D f(x_k) \xi_k,
\label{eq.wolfe}
\end{equation}
where $c_2>c_1$ is a constant smaller than 1 (around $0.9$ for LBFGS). Practical line-search algorithms usually satisfy \emph{strong Wolfe conditions}, where~\eqref{eq.wolfe} is replaced by the stronger condition:
\begin{equation*}
|D f(\retr_{x_k}(\alpha \xi_k)) \mathcal{T}_{x_k,\retr_{x_k}(\alpha \xi_k)}(\xi_k)| \leq c_2 |D f(x_k) \xi_k|.
\end{equation*}
\begin{minipage}[h]{1.0\linewidth}\footnotesize
  \begin{minipage}[t]{0.5\linewidth}\footnotesize
    \begin{algorithm}[H]
      \caption{\footnotesize Wolfe line-search}
      \label{alg.wls}
      \begin{algorithmic}[1]\raggedright\footnotesize
        \State{\bf Given:} Current point $x_k$ and descent direction $\xi_k$
        \State $\phi(\alpha) \gets f(R_{x_k}(\alpha \xi_k))$;  $\phi^{\prime}(\alpha) \gets \alpha D f(x_k) \xi_k$
        \State $\alpha_0 \gets 0$, $\alpha_1 > 0$ and $i\gets 0$.
        \While{$i \leq i_{\max}$}
          \State $i \leftarrow i+1$
          \If { $\phi(\alpha_i) > \phi(0) + c_1 \alpha_i \phi^{\prime}(0)$ \textbf{or}  $\phi(\alpha_i) \geq \phi(\alpha_{i-1}),\ i>1$}
            \State $\alpha_{\text{low}}=\alpha_{i-1}$ and $\alpha_{\text{hi}}=\alpha_{i}$ 
            \State {\bf break}
          \ElsIf { $|\phi^{\prime}(\alpha_i)| \leq c_2 \phi^{\prime}(0)$} 
            return $\alpha_i$
          \ElsIf { $|\phi^{\prime}(\alpha_i)| \geq 0$} 
            \State $\alpha_{\text{low}}=\alpha_{i}$ and $\alpha_{\text{hi}}=\alpha_{i-1}$
            \State {\bf break}
          \Else
            \State \textsc{Extrapolate} to find $\alpha_{i+1} > \alpha_i$
          \EndIf
        \EndWhile
        \State \textbf{Call} \textsc{ZoomingPhase}
      \end{algorithmic}
    \end{algorithm}
  \end{minipage}
  \begin{minipage}[t]{0.5\linewidth}\footnotesize
    \begin{algorithm}[H]
      \caption{\footnotesize \textsc{ZoomingPhase}}
      \label{alg.zooming}
      \begin{algorithmic}[1]\raggedright\footnotesize
        \While{$i \leq i_{\max}$} 
        \State $i \leftarrow i+1$
        \State \textsc{Interpolate} to find $\alpha_i \in (\alpha_{\text{low}}, \alpha_{\text{hi}})$
        \If { $\phi(\alpha_i) > \phi(0) + c_1 \alpha_i \phi^{\prime}(0)$ \textrm{\bf or} $\phi(\alpha_i) \geq \phi(\alpha_{\text{low}})$} 
        \State $\alpha_{\text{hi}} \leftarrow \alpha_i$
        \Else
        \If { $|\phi^{\prime}(\alpha_i)| \leq c_2 \phi^{\prime}(0)$} 
        \Return $\alpha_i$
        \ElsIf { $\phi^{\prime}(\alpha_i)(\alpha_{\text{hi}}-\alpha_{\text{low}}) \geq 0 $} 
        \State $\alpha_{\text{hi}} \leftarrow \alpha_{\text{low}}$
        \EndIf
        \State $\alpha_{\text{low}} \leftarrow \alpha_i$
        \EndIf
        \EndWhile 

        \State \textbf{return} {\bf failure}
      \end{algorithmic}
    \end{algorithm}
  \end{minipage}
\end{minipage}
\vskip7pt
Algorithm~\ref{alg.wls} summarizes a line-search algorithm satisfying strong Wolfe conditions based on the Euclidean algorithm explained in~\citep{nocedal2006numerical}. The algorithm is divided into two phases: bracketing and zooming. In the bracketing phase, an interval is found that contains a point satisfying the strong Wolfe condition. Next, in the zooming phase, the actual point is found. Theory behind why this algorithm is guaranteed to find a step-length satisfying (strong) Wolfe conditions can be found in \citep{nocedal2006numerical}.  

For the interpolation and extrapolation steps of the line-search one can find the minimum of a cubic polynomial approximation to the function in an interval.  For cubic polynomial interpolation, we approximate the function by a cubic polynomial so that the function $\phi(\cdot)$ and its gradient $\phi^{\prime}(\cdot)$ matches the function value and the gradient of the cubic polynomial at the end-points of the interval. For extrapolation, we use the function and gradient at $0$ and at the end-point.
To ensure numerical stability, the interval wherein the minimum of the cubic polynomial is computed in the interpolation phase is chosen to be smaller than the actual interval so to have certain distances from the end-points of the interval (we choose the distance to be 0.1 times the interval length). The interval for the extrapolation is assumed to be between $1.1$ and $10$ times the value of the current point.

The initial step-length $\alpha_1$ can be guessed using the previous function and gradient information. We propose the following choice that is quite effective:
\begin{equation}
\label{eq.alpha1}
\alpha_1 = 2\frac{f(x_k) - f(x_{k-1})}{  D f(x_k) \xi_k}.
\end{equation}
Equation~\eqref{eq.alpha1} is obtained by finding $\alpha^*$ that minimizes a quadratic approximation of the function along the geodesic through the previous point (based on $f(x_{k-1})$, $f(x_k)$ and $D f(x_{k-1}) \xi_{k-1}$):
\begin{equation}
\label{eq.quadappr}
\alpha^* = 2\frac{f(x_k) - f(x_{k-1})}{  D f(x_{k-1}) \xi_{k-1}}.
\end{equation}
Then, assuming that first-order change will be the same as in the previous step, we write
\begin{equation}
\label{eq.samefirst}
\alpha^*  D f(x_{k-1}) \xi_{k-1} \approx \alpha_1  D f(x_{k}) \xi_{k}.
\end{equation}
Combining~\eqref{eq.quadappr} and~\eqref{eq.samefirst}, we obtain our procedure of selection $\alpha_1$ expressed in~\eqref{eq.alpha1}. \citet{nocedal2006numerical} suggest using either $\alpha^*$ of \eqref{eq.quadappr} as the initial step-length, or using~\eqref{eq.samefirst} where $\alpha^*$ is set equal to the step-length obtained in the line-search at the previous point. We observed the our choice \eqref{eq.alpha1} proposed above leads to substantially better performance than these other two approaches.

\subsection{Stochastic optimization}
If the objective function has the form
\begin{equation}
  \label{eq:8}
  \min_{X \in \Mc}\quad f(x) := \frac{1}{n}\nlsum_{i=1}^nf_i(x),
\end{equation}
then for large $n$ each iteration of the first-order methods explained above becomes very expensive, as merely computing the gradient requires going through all $n$ component functions. In this large-scale setting, one frequently passes to stochastic / incremental optimization methods such as stochastic gradient descent (SGD) that processes only a small batch of functions at each iteration. Note that SGD is actually not a descent method; it makes progress by replacing an exact descent direction by one which is a descent direction in expectation.

Riemannian SGD~\citep{bonnabel2013} runs the following iteration, where $i_t \sim U(n)$, i.e. a random integer between $1$ and $n$:
\begin{equation}
  \label{eq:10}
  x_{t+1} \gets \retr_{x_t}(-\eta_t \nabla f_{i_t}(x_t)),\qquad t=0,1,\ldots,
\end{equation}
where $\retr_x$ is a retraction at the point $x$ and $\eta_t$ is a suitable stepsize that typically satisfies $\sum_t\eta_t=\infty$ and $\sum_t\eta_t^2<\infty$. 

After this background on the Riemannian optimization methods that we will use for \gmm parameter optimization, we are now ready to describe the problem reformulation and other important theoretical details.

\section{Problem reformulation}
\label{sec:prob}
Experience with mixture modeling shows that whenever an optimization method works well for a single component, the same optimization method also works well for the mixture model. We begin, therefore, with parameter estimation for a single Gaussian. Although this problem has a closed-form solution that benefits \EM, our goal is to tackle it in the context of manifold optimization.

Consider, maximum likelihood parameter estimation for a single Gaussian,
\begin{equation}
  \label{eq:3}
  \max_{\vmu,\msigma \succ 0}\ \Lc(\vmu,\msigma) := \nlsum_{i=1}^n\log p_{\Nc}(\vx_i; \vmu, \msigma).
\end{equation}
This objective is concave in the Euclidean sense. But our aim is to apply manifold optimization and this objective is \emph{not} g-concave on its domain $\reals^d\times\pp^d$, which makes it geometrically somewhat of a mismatch. 

We invoke a simple transformation that turns \eqref{eq:3} into a g-concave optimization problem. This transformation has a dramatic impact on the speed of the convergence for a single Gaussian, as seen in Fig.~\ref{fig:reparam}.  Define new vectors $\vy_i^T = [\vx_i^T\ 1]$; then, the proposed transformed model is
\begin{equation}
  \label{eq:5}
  \max_{\ms \succ 0}\ \widehat{\Lc}(\ms) := \nlsum_{i=1}^n \log q_{\Nc}(\vy_i;\ms),
\end{equation}
where $q_{\Nc}(\vy_i;\ms) := 2\pi \exp(\tfrac 12) p_\mathcal{N}(\vy_i;\ms)$. Note that this new cost function is \emph{not} just a reparametrization of~\eqref{eq:3}. However, it becomes a reparametrization at a maximum. More precisely, Theorem~\ref{thm.gauss} shows that solving the reformulation~\eqref{eq:5} also solves the original problem~\eqref{eq:3}.

\begin{theorem} 
  \label{thm.gauss}
  If $\vmu^*, \msigma^*$ maximize~\eqref{eq:3}, and if $\ms^*$ maximizes~\eqref{eq:5}, then $\widehat{\Lc}(\ms^*) = \Lc(\vmu^*, \msigma^*)$ and
  \begin{equation}
    \label{eq.sstar}
    \ms^* = \begin{pmatrix}
    \msigma^* + \vmu^* {\vmu^*}^T & \vmu^* \\
    {\vmu^*}^T & 1
  \end{pmatrix}.
\end{equation}
\end{theorem}
\begin{proof}
We express $\ms$ by new variables $\mU$, $\vt$ and $s$ by writing 
\begin{equation}
\label{eq.s}
\ms = \begin{pmatrix}
    \mU + s \vt \vt^T & s\vt \\
   s\vt^T &s
  \end{pmatrix}.
  \end{equation}
  The objective function $\widehat{\Lc}(\ms)$ in terms of the new parameters becomes
  \begin{equation*}
  \begin{split}
      \widehat{\Lc}(\mU,\vt,s) =\tfrac{n}{2} -\tfrac{d}{2}\log(2\pi)  - \tfrac{n}{2} \log s -\tfrac{n}{2}\log\det(\mU)\\ - \sum_{i=1}^n \tfrac12(\vx_i-\vt)^T \mU^{-1} (\vx_i-\vt) - \tfrac{n}{2s}.
   \end{split}
  \end{equation*}
  Optimizing $\widehat{\Lc}$ over $s>0$ we see that $s^{*}=1$. Hence, the objective reduces to a $d$-dimensional Gaussian log-likelihood, for which $\mU^*=\msigma^*$ and $\vt^*=\vmu^*$.
\end{proof}

In other words, Theorem~\ref{thm.gauss} shows that our model transformation is ``faithful'' because it leaves the optimum unchanged. Figure~\ref{fig:reparam} shows the unmistakable impact this transformation has on the convergence speed of Riemannian Conjugate-Gradient (CG) and Riemannian LBFGS. 

Next, Proposition~\ref{prop:gc} proves another key property of this transformation: the objective in~\eqref{eq:5} becomes g-concave. For proving Proposition~\ref{prop:gc}, we need the following lemma that is an easy consequence of \citep[Thm.~4.1.3]{bhatia07}:
\begin{lemma}
  \label{lem:gm}
  Let $\ms$, $\mr \succ 0$. Then, for a vector $\vx$ of appropriate dimension,
  \begin{equation}
\label{eq:15}
    \vx^T(\ms^{-\nhalf}(\ms^{\nhalf}\mr^{-1}\ms^{\nhalf})^{\nhalf}\ms^{-\nhalf})\vx \le [\vx^T\ms^{-1}\vx]^{\nhalf}[\vx^T\mr^{-1}\vx]^{\nhalf}.
  \end{equation}
\end{lemma}
\begin{proposition}
  \label{prop:gc}
  The objective $\widehat{\Lc}(\ms)$ in~\eqref{eq:5} is g-concave.
\end{proposition}
\begin{proof}
  By continuity, it suffices to establish mid-point geodesic concavity:
  \begin{equation*}
    \widehat{\Lc}(\gamma_{\ms,\mr}(\half)) 
    \ge \half\widehat{\Lc}(\ms) + \half\widehat{\Lc}(\mr),\qquad\text{for}\ \ms,\mr \in \pp^d.
  \end{equation*}
  Denoting inessential constants by $c$, the above inequality turns into 
  \begin{align*}
    \widehat{\Lc}(\gamma_{\ms,\mr}(\half&)) 
    = -\log\det(\ms^{\nhalf}\mr^{\nhalf}) - c\sum_i\vy_i^T(\ms^{-\nhalf}(\ms^{\nhalf}\mr^{-1}\ms^{\nhalf})^{\nhalf}\ms^{-\nhalf})\vy_i\\
    &\ge -\half\log\det(\ms) -\half\log\det(\mr) - c\sum_i [\vy_i^T\ms^{-1}\vy_i]^{\nhalf}[\vy_i^T\mr^{-1}\vy_i]^{\nhalf}\\
    &\ge -\half\log\det(\ms) - \tfrac c2\sum_i \vy_i^T\ms^{-1}\vy_i -\half\log\det(\mr) - \tfrac c2\sum_i \vy_i^T\mr^{-1}\vy_i\\
    &=\half \widehat{\Lc}(\ms) + \half \widehat{\Lc}(\mr),
  \end{align*}
  where the first inequality is follows from Lemma~\ref{lem:gm}.
\end{proof}

\begin{theorem}
  \label{thm:gmm.reparam}
  A local maximum of the reformulated \gmm log-likelihood
  \begin{equation*}
    \widehat{\Lc}(\{\ms_j\}_{j=1}^K) 
    := \nlsum_{i=1}^n \log\Bigl(\nlsum_{j=1}^K\alpha_j q_{\Nc}(\vy_i; \ms_j)\Bigr)
  \end{equation*}
  is a local maximum of the original log-likelihood
  \begin{equation*}
    \Lc(\{\vmu_j,\msigma_j\}_{j=1}^K) := 
    \nlsum_{i=1}^n \log\Bigl(\nlsum_{j=1}^K\alpha_jp_{\Nc}(\vx_i|\vmu_j,\msigma_j)\Bigr).
  \end{equation*}
\end{theorem}
\begin{proof}
  Let $\ms_1^*, \ldots, \ms_K^*$ be a local maximum of $\widehat{\Lc}$. Then, $\ms_j^*$ is the maximum of the following cost function:
  \begin{equation*}
   - \frac12 \nlsum_{i=1}^n w_i \log\det(\ms_j) - \frac12\nlsum_{i=1}^n w_i \vy_i^T \ms_j^{-1} \vy_i,
  \end{equation*} 
  where for each $i \in \set{1,\ldots,n}$ the weight
  \begin{equation}
  \label{eq.weight}
    w_i = \frac{q_{\Nc}(\vy_i|\ms^*_j)}{\sum_{j=1}^K\alpha_j q_{\Nc}(\vy_i|\ms^*_j) }.
  \end{equation}
  Using an argument similar to that for Theorem~\ref{thm.gauss}, we see that $s_j^*=1$, whereby $q_{\Nc}(\vy_i|\ms^*_j) = p_{\Nc}(\vx_i;\vt_j^*,\mU_j^*)$. Thus, at a maximum the objective functions agree and the proof is complete.
\end{proof}

Theorem~\ref{thm:gmm.reparam} shows that we can replace~\eqref{eq:2} by a reformulated log-likelihood whose local maxima agree with those of~\eqref{eq:2}. Moreover, the individual components of the reformulated log-likelihood are geodesically concave. 

Finally, we also need to replace the constraint $\valpha \in \Delta_K$ to make the problem unconstrained. We do this via a commonly used change of variables~\citep{jordan1994hierarchical}:
\begin{equation}
\omega_k =  \log \biggl ( \frac{ \alpha_k}{\alpha_K} \biggr),\quad k=1,\hdots,K-1.
\label{eq.wReparam}
\end{equation}
Assume $\omega_K=0$ to be a constant; then the final optimization problem is:
\begin{equation}
  \label{eq:6}
  \max_{\{\ms_j \succ 0\}_{j=1}^K,\{\omega_j\}_{j=1}^{K-1}}  \widehat{\Lc}(\{\ms_j\}_{j=1}^K,\{\omega_j\}_{j=1}^{K-1}) 
  := \sum_{i=1}^n \log\Bigl(\sum_{j=1}^K\tfrac{\exp(\omega_j)}{\sum_{k=1}^K \exp(\omega_k)} q_{\Nc}(\vy_i; \ms_j)\Bigr)
\end{equation}
We solve~\eqref{eq:6} via Riemannian optimization problem in this paper; specifically, it is an optimization problem on the product manifold $\bigl(\prod_{j=1}^K\pp^{d+1}\bigr) \times \reals^{K-1}$. 

\subsection{Formulations for Penalized Likelihoods}
\label{sec:priors}
One of the problems with ML estimation for \gmm{}s is covariance singularity. There are several remedies to avoid this problem, and the most common approach is to use a penalized ML estimate~\citep{ridolfi2001penalized}. We state the following generic results that helps choose priors amenable to our framework.

\begin{theorem}
\label{thm.pen}
Let $\ms$ be the block matrix defined in~\eqref{eq.s}. Consider a regularizer that splits over the blocks of $\ms$, and has the form
\begin{equation*}
  \psi(\ms) = \psi_1(\mU,\vt) + \psi_2(s),
\end{equation*}
where $\psi_2(s)$ has a unique maximizer at $s=1$. Let $\ms^*$ be the maximum of the penalized objective $\psi(\ms) + \widehat{\Lc}(\ms)$, where $\widehat{\Lc}(\ms)$ is the modified log-likelihood~\eqref{eq:5}. Assume that $(\vmu^*,\msigma^*)$ maximizes the penalized log-likelihood $\psi_1(\msigma,\vmu) +\Lc(\vmu,\msigma)$, where $\Lc(\vmu,\msigma)$ is as in~\eqref{eq:3}. Then, $\ms^*$  is related to $(\vmu^*,\msigma^*)$ via~\eqref{eq.sstar}.
\end{theorem}
\begin{proof}
Similar to the proof of Theorem~\ref{thm.gauss}, it is easy to see that the penalized objective $\psi+\widehat{\Lc}$ has its maximum at $s^*=1$. Therefore, the objective reduces to a penalized log-likelihood of a Gaussian at its maximum.
\end{proof}

A widely used penalizer is obtained by placing an inverse Wishart prior on covariance matrices and using a   \emph{maximum a priori} estimate. The inverse Wishart prior is a conjugate prior for the covariance matrix, and is given by
\begin{equation*}
p(\msigma;\mlambda;\nu) \propto \det(\msigma)^{-(\nu+d+1)/2} \exp \bigl ( -\tfrac12 \trace (\msigma^{-1} \mlambda) \bigr ),
\end{equation*}
where $\nu$ is a degree of freedom and $\mlambda$ is a scale parameter. The conjugate prior for the mean parameter is a Gaussian distribution conditioned on the covariance matrix; that is, 
\begin{equation*}
p(\vmu|\msigma;\vlambda,\kappa) \propto \det(\msigma)^{-1/2}  \exp \bigl (-\tfrac \kappa2 (\vmu-\vlambda)^T \msigma^{-1} (\vmu-\vlambda) \bigr ),
\end{equation*}
where $\kappa$ is a so-called shrinkage parameter. 

In the following, we propose a penalizer to our reformulated objective function. This penalized objective function converges to the penalized log-likelihood for \gmm, when one uses the aforementioned conjugate priors for covariance matrices and means. 

\noindent Consider the penalizer
\begin{equation}
\label{eq.pen1}
\psi(\ms;\mpsi) = -\frac{\rho}{2}\log\det(\ms)-\beta\tfrac12  \trace(\mpsi \ms^{-1}),
\end{equation} 
where $\mpsi$ is the block matrix
\begin{equation}
\mpsi = \begin{pmatrix}
     \frac{\alpha}{\beta}\mlambda+\kappa \vlambda\vlambda^T  &\kappa \vlambda \\
\kappa \vlambda^T &\kappa 
  \end{pmatrix},
  \end{equation}
   and the parameter $\rho = \alpha(d+\nu+1)+\beta$. If we write $\ms$ as the block matrix
\begin{equation*}
\ms = \begin{pmatrix}
    \mU + s \vt \vt^T & s\vt \\
   s\vt^T &s
  \end{pmatrix},
  \end{equation*}
then the penalized cost function~\eqref{eq.pen1} becomes
\begin{equation*}
\begin{split}
\psi(\ms;\mpsi) &= -\frac{\rho}{2}\bigl [ \log\det(\mU) +\log(s) \bigr ]\\
&-\tfrac \beta2 \left[ \tfrac{\alpha}{\beta}\trace (\mlambda \mU^{-1}) + \kappa (\vt^T \mU^{-1} \vt) + \kappa (\vlambda^T \mU^{-1} \vlambda) -2\kappa  \vlambda^T \mU^{-1} \vt + \tfrac{\kappa}{s}  \right].
\end{split}
\end{equation*}  
Rearranging the terms, we thus obtain
\begin{equation}
\psi(\ms;\mpsi) = \alpha\log p(\mU;\mlambda;\nu) + \beta\log p(\vt|\mU;\vlambda,\kappa) - \frac{\rho}{2} \log(s) -\frac{\beta\kappa}{2s}+c,
\label{eq.penalizer}
\end{equation}  
for some constant $c$. In order for this penalizer to satisfy the conditions of Theorem~\ref{thm.pen} we need the following condition:
\begin{equation*}
\alpha =\beta \frac{\kappa-1}{d+\nu+1}.
\end{equation*}
Using Proposition~\ref{prop:gc} one can again show that this penalizer is g-concave. We summarize these results as an informal corollary below.

\begin{corollary}
The penalizer given in~\eqref{eq.penalizer} is g-concave and fulfills the structure required by Theorem~\ref{thm.pen}. Hence, it can be used for penalized ML estimation.
\end{corollary}

It is easy to see that the single component results above extend to penalized maximum likelihood of \gmm{}s. That is, Theorem~\ref{thm:gmm.reparam} can be generalized to penalized maximum likelihood for \gmm{}s.  

Indeed, recall that a common prior on mixture weights is the symmetric Dirichlet prior that assumes the form
\begin{equation}
  \label{eq:16}
p(\alpha_1,\hdots,\alpha_K;\zeta) \propto \prod_{i=1}^{K} \alpha_i^{\zeta}.
\end{equation}
The penalizer for the mixture weights is the logarithm of \eqref{eq:16}, namely,
\begin{equation}
\label{eq.priorweight}
\varphi(\{\omega_j\}_{i=1}^{K-1};\zeta) := \zeta \sum_{i=1}^{K} \log\left( \tfrac{e^{\omega_j}}{\sum_{k=1}^K e^{\omega_k}}\right)=  \zeta \sum_{i=1}^{K} \omega_i - K \zeta \log\Bigl(\sum_{k=1}^K e^{\omega_k}\Bigr).
\end{equation}
The final optimization problem for the penalized mixture model is
\begin{equation}
 \max_{\{\ms_j \succ 0\}_{j=1}^K,\{\omega_j\}_{j=1}^{K-1}}  \widehat{\Lc}(\{\ms_j\}_{j=1}^K,\{\omega_j\}_{j=1}^{K-1}) + \sum_{j=1}^{K}\psi(\ms_j;\mpsi)  + \varphi(\{\omega_j\}_{i=1}^{K-1};\zeta), 
\end{equation}
where $\widehat{\Lc}(\{\ms_j\}_{j=1}^K,\{\omega_j\}_{j=1}^{K-1})$, $\psi(\ms;\mpsi)$ and $\varphi(\{\omega_j\}_{i=1}^{K-1};\zeta)$ are given by~\eqref{eq:6}, \eqref{eq.pen1}, and \eqref{eq.priorweight}, respectively. 

\vskip12pt
\noindent We have now presented our formulation of the main optimization problems of this paper, both \gmm fitting, as well as a penalized version based on using an conjugate priors on means and covariance matrices combined with a Dirichlet model for mixture components weights. We can solve both these problems using Riemannian LBFGS procedure or a Riemannian SGD method for larger scale problems. The former method was also studied in~\citep{hosseini2015matrix}; we thus dedicate Section~\ref{sec:sgd} to an general analysis Riemannian SGD before specializing it to our  \gmm problems in Section~\ref{sec:sgd.gmm}.
\section{Riemannian stochastic optimization}
\label{sec:sgd}

In this section, we consider the stochastic gradient descent algorithm
\begin{equation}
  \label{eq:7}
  x_{t+1} \gets \retr_{x_t}(-\eta_t \nabla f_{i_t}(x_t)),\qquad t=0,1,\ldots,
\end{equation}
where $\retr_x$ is a suitable retraction (to be specialized later).
We assume for our analysis of~\eqref{eq:7} the following fairly standard conditions:
\begin{list}{--}{\leftmargin=2.5em}
\setlength{\itemsep}{-1pt}
\item[(i)] The function satisfies the Lipschitz growth bound 
\begin{equation}
  f(\retr_x(\xi)) \le f(x) + \ip{\nabla f(x)}{\xi} + \tfrac{L}{2}\norm{\xi}^2.
 \label{eq.condretr} 
\end{equation}
 \item[(ii)] The stochastic gradients in all iterations are unbiased, i.e.,
 \begin{equation*}
 \E[\nabla f_{i_t}(x_t) - \nabla f(x_t)] = 0.
 \end{equation*}
 \item[(iii)] The stochastic gradients have bounded variance, so that 
  \begin{equation*}
 \E[\norm{\nabla f_{i_t}(x_t) - \nabla f(x_t)}^2] \leq \sigma^2,\qquad 0\le \sigma < \infty.
 \end{equation*}
\end{list}
When the retraction is the exponential map, condition (i) can be reexpressed as (provided that $\Exp_y^{-1}(\cdot)$ exists)
\begin{equation}
  \label{eq:9}
  f(x) - f(y) - \ip{\nabla f(y)}{\Exp^{-1}_{y}(x)} \le \tfrac{L}{2}d^2(x,y).
\end{equation}
Given these conditions, the iterates produced by~\eqref{eq:7} satisfy the following:
\begin{lemma}
\label{lem.sgd1}
Assume conditions (i)-(iii) hold. Then, the gradients in SGD satisfy the bound
\begin{equation}
    \label{eq:14}
    \sum_{t=1}^T\left(\eta_t^2-\tfrac{L}{2}\eta_t^2\right)\E[\norm{\nabla f(x_t)}^2]
    \le
    f(x_1) - f^* + \tfrac{L\sigma^2}{2}\nlsum_{t=1}^T\eta_t^2. 
  \end{equation}
\end{lemma}
\begin{proof}
  Denote the stochastic error by $\delta_t = \nabla f(x_t) - \nabla f_{i_t}(x_t)$; also, as a shorthand set $g_t=\nabla f_{i_t}(x_t)$. Then, we have
  \begin{align*}
    f(x_{t+1}) &\le f(x_t) + \ip{\nabla f(x_t)}{-\eta_t\nabla f_{i_t}(x_t)} + \tfrac{L}{2}\norm{\eta_t\nabla f_{i_t}(x_t)}^2\\
    &= f(x_t) - \eta_t\ip{\nabla f(x_t)}{g_t} + \tfrac{L\eta_t^2}{2}\norm{g_t}^2\\
    &= f(x_t) - \eta_t\norm{\nabla f(x_t)}^2 - \eta_t\ip{\nabla f(x_t)}{\delta_t} +  
     \tfrac{L\eta_t^2}{2}\bigl[\norm{\nabla f(x_t)}^2 + 2\ip{\nabla f(x_t)}{\delta_t} + \norm{\delta_t}^2\bigr]\\
    &= f(x_t) - \left(\eta_t^2-\tfrac{L}{2}\eta_t^2\right)\norm{\nabla f(x_t)}^2  - \bigl(\eta_t-L\eta_t^2\bigr)\ip{\nabla f(x_t)}{\delta_t} + \tfrac{L\eta_t^2}{2}\norm{\delta_t}^2.
  \end{align*}
  Summing over $t=1,\ldots,T$, using telescoping sums and rearranging we obtain
  \begin{align*}
    &\sum_{t=1}^T\left(\eta_t^2-\tfrac{L}{2}\eta_t^2\right)\norm{\nabla f(x_t)}^2 \\
    &\qquad \qquad\quad\le f(x_{1})-f(x_{T+1}) - \sum_{t=1}^T \bigl(\eta_t-L\eta_t^2\bigr)\ip{\nabla f(x_t)}{\delta_t} + \frac{L}{2}\sum_{t=1}^T\eta_t^2\norm{\delta_t}^2\\
    &\qquad \qquad\quad\le f(x_{1})-f^* - \sum_{t=1}^T \bigl(\eta_t-L\eta_t^2\bigr)\ip{\nabla f(x_t)}{\delta_t} + \frac{L}{2}\sum_{t=1}^T\eta_t^2\norm{\delta_t}^2,
  \end{align*}
  where we used $f^* \le f(x_t)$ for all $t$. Now taking expectations, and noting that by our assumption $\E[\norm{\delta_t}^2] \le \sigma^2$ while by unbiasedness of the stochastic gradients we have $\E[\ip{\nabla(x_t)}{\delta_t}]=0$. Thus, we obtain the bound~\eqref{eq:14}.
\end{proof}
By using a specific choice of parameter $\eta_t$ and using Lemma~\ref{lem.sgd1}, we can obtain a convergence rate result for SGD with a slight modification.
\begin{theorem}
\label{thm.sgdcond23}
Assume a slightly modified version of SGD which output a point $x_a$ by  randomly picking one of the iterates, say $x_t$, with probability $p_t := (2\eta_t-L\eta_t^2)/Z_T$, where $Z_T=\sum_{t=1}^T(2\eta_t-L\eta_t^2)$. Furthermore, choose $\eta_t = \min\set{L^{-1}, c\sigma^{-1}T^{-1/2}}$ for a suitable constant $c$. Then, we obtain the following bound on $\E[\norm{\nabla f(x_a)}^2]$, which measures the expected gap to stationarity:
  \begin{equation}
  \label{eq:bound}
    \E[\norm{\nabla f(x_a)}^2] \le \frac{2L\Delta_1}{T} + \bigl(c+c^{-1}\Delta_1\bigr)\frac{L\sigma}{\sqrt{T}} = \Oc\left(\frac1T\right)+\Oc\left(\frac{1}{\sqrt{T}}\right).
  \end{equation}
\end{theorem} 
\begin{proof}
  Using the definition of $x_a$ and using Lemma~\ref{lem.sgd1}, we immediately have
  \begin{align*}
    \E[\norm{\nabla f(x_a)}^2] = \sum_{t=1}^T p_t\E[\norm{\nabla f(x_t)}^2] \le \frac{2(f(x_1)-f^*)}{Z_T}  + L\sigma^2\frac{\nlsum_{t=1}^T\eta_t^2}{Z_T}.
  \end{align*}
  Using the choice of $\eta_t$ in the theorem, this bound yields~\eqref{eq:bound}.
\end{proof}

Theorem~\ref{thm.sgdcond23} uses a randomized stopping rule, a choice motivated by~\citep{Ghadimi13}. If one wishes to avoid such a rule, then under a stronger assumption one can obtain the same rate. Specifically, in the theorem below we replace conditions (ii) and (iii) with the stronger condition (iv).
\begin{list}{--}{\leftmargin=2.5em}
\setlength{\itemsep}{-1pt}
\item [(iv)] The function $f$ has a $G$-bounded gradient, that is $\norm{\nabla f_i(x)} \le G$ for all $i \in [n]$
\end{list}
Under this condition, we can obtain the following convergence rate.
\begin{theorem}
\label{thm.sgdcond4}
  Assume conditions (i) and (iv) hold. Then, the gradient in SGD satisfies the following bound for a suitable choice of  $\eta_t$:
  \begin{equation}
    \label{eq:11}
    \frac{1}{T}\sum_{t=1}^{T}\E[\norm{\nabla f(x_t)}^2] \le \frac{1}{\sqrt{T}}\left( \frac{f(x_1)-f(x_*)}{c} + \frac{Lc}{2}G^2\right).
  \end{equation}
\end{theorem}
\begin{proof}
  The Lipschitz smoothness condition yields
  \begin{align*}
    \E[f(x_{t+1})] &\le \E[f(x_t)] + \E\Bigl[\ip{\nabla f(x_t)}{-\eta_t\nabla f_{i_t}(x_t)} + \tfrac{L}{2}\norm{\eta_t\nabla f_{i_t}(x_t)}^2\Bigr]\\
    &\le \E[f(x_t)] - \eta_t\E\bigl[\norm{\nabla f(x_t)}^2\bigr] + \tfrac{L\eta_t^2}{2}G^2.
  \end{align*}
  Rearranging the terms above we obtain
  \begin{align*}
    \E\bigl[\norm{\nabla f(x_t)}^2\bigr] \le \frac{1}{\eta_t}\E\Bigl[f(x_t)-f(x_{t+1}) \Bigr] + \frac{L\eta_t}{2}G^2.
  \end{align*}
  Choose $\eta_t=\frac{c}{\sqrt{T}}$ for some constant $c$ and sum over $t=0$ to $T-1$ to obtain
  \begin{align*}
    \frac{1}{T}\sum_{t=1}^{T}\E\bigl[\norm{\nabla f(x_t)}^2\bigr]
    &\le
      \frac{1}{\sqrt{T}c}\E[f(x_1)-f(x_{T+1})] + \frac{Lc}{2\sqrt{T}}G^2\\
    &\le
      \frac{1}{\sqrt{T}}\left(\frac{f(x_1)-f(x^*)}{c} + \frac{Lc}{2}G^2\right).~\hskip1in
  \end{align*}
\end{proof}
By optimizing over the constant $c$, the following corollary is immediate.
\begin{corollary}
  Assume conditions (i) and (iv) hold, then for suitable $\eta_t$ we have
  \begin{equation}
    \label{eq:12}
    \min_{1\le t\le T}\E[\norm{\nabla f(x_t)}^2] \le \Oc\left(\frac{1}{\sqrt{T}}\right).
  \end{equation}
\end{corollary}

\section{SGD for GMM}
\label{sec:sgd.gmm}
In this section, we investigate if SGD based on retractions satisfies the conditions needed for obtaining a global rate of convergence when applied to our \gmm optimization problems. Since Euclidean retraction turns out to be computationally more effective than many other retractions, we perform the analysis below for Euclidean retraction.

Recall that we are maximizing a cost of the form $\frac1n\sum_{i=1}^nf_i(\cdot)$ using SGD. In a concrete realization, each function $f_i$ is set to the penalized log-likelihood for a batch of observations (data points). For simpler notation, assume that each $f_i$ corresponds to a single observation. Thus,
\begin{equation}
\begin{split}
f_i(\{\ms_j \succ 0\}_{j=1}^K,\{\eta_j\}_{j=1}^{K-1}) =  \log\Bigl(\sum_{j=1}^K\frac{\exp(\eta_j)}{\sum_{k=1}^K \exp(\eta_k)} q_{\Nc}(\vy_i; \ms_j)\Bigr) \\ 
+ \frac1n \biggl (\sum_{j=1}^{K}\psi(\ms_j;\mpsi)  + \varphi(\{\eta_j\}_{i=1}^{K-1};\zeta) \biggr ),
\end{split}
\label{eq.objPenSGD}
\end{equation}  
where $q_{\Nc}$, $\psi$ and $\varphi$ are as defined by~\eqref{eq.pen1} and~\eqref{eq.priorweight}, respectively. Since we are maximizing, the update formula for SGD is
\begin{equation}
 \Bigl\{\set{\ms_j \succ 0}_{j=1}^K, \set{\eta_j}_{j=1}^{K-1}\Bigr\} \  \gets \retr_{\{\ms_j \succ 0\}_{j=1}^K,\{\eta_j\}_{j=1}^{K-1}}\left( \eta_t \nabla f_{i}\left( \{\ms_j \succ 0\}_{j=1}^K,\{\eta_j\}_{j=1}^{K-1} \right) \right),
 \label{eq.updategmm}
\end{equation}
where $i$ is a randomly chosen index between $1$ and $n$.

Note that, the conditions needed for a global rate of convergence are not satisfied on the entire set of positive definite matrices. In particular, to apply our convergence results for SGD we need to show that the iterates stay within a compact set. Theorem~\ref{thm:compact} below ensures this property.
\begin{theorem}
\label{thm:compact}
If the stepsize is smaller than one, then the iterates of SGD for the penalized likelihood of \gmm stay within a compact set.
\end{theorem}
\begin{proof}
We write down the formula of the gradient and show that the update formula~\eqref{eq.updategmm} guarantees that the variables remain in a bounded set. The Euclidean gradient of penalized log-likelihood with respect to one of the covariance matrices $\ms_j$ for a single datapoint $\vy_i$ is equal to
\begin{equation}
\nabla_E f_i(\ms_j) = -\frac w2\ms_j^{-1} + \frac w2 \ms_j^{-1}  \vy_i \vy_i^T  \ms_j^{-1} -\frac{\rho}{2n} \ms_j^{-1}+\frac{\beta}{2n}\ms_j^{-1}\mpsi \ms_j^{-1},
\end{equation}
where $w$, a weight calculated as in~\eqref{eq.weight}, is a positive number smaller than 1 and $\rho$, a small constant that appears in $\psi(\ms;\mpsi)$, is of order of $10^{-2}$. Using the update formula~\eqref{eq.updategmm}, $\ms_j$ is updated by
\begin{equation}
\label{eq.updategmm2}
\ms_j \leftarrow \biggl(1 - \eta_t \frac{w+\rho n^{-1}}{2}\biggr) \ms_j
+ \eta_t \mpsi', 
\end{equation}
where
\begin{equation*}
\mpsi' = \frac{w}{2} \vy_i \vy_i^T +\frac{\beta n^{-1}}{2}\mpsi . 
\end{equation*}
If $\eta_t\leq1$, then the first term in~\eqref{eq.updategmm2} remains positive definite.  Assume $\lambda$ and $\lambda'$ to be the smallest eigenvalue of $\ms_j$ before and after the update of~\eqref{eq.updategmm2}. Furthermore, assume the smallest eigenvalue of $\ms_j$ before update be $\lambda_{\min}(\ms_j) = \tau \lambda_{\min}(\mpsi)$. From the update rule~\eqref{eq.updategmm2} and knowing that the smallest eigenvalue of sum of two matrices with positive eigenvalues is not smaller than sum of smallest eigenvalue of two matrices, we have
\begin{equation*}
\lambda' \geq \lambda + \frac{\eta_t}{2}\lambda_{\min}(\mpsi) \biggl(-\tau(w+\rho n^{-1}) + \beta n^{-1} \biggr).
\end{equation*}
If $ \tau < \beta/(n+\rho)$, then $\lambda' > \lambda$. Otherwise, $\lambda' \geq \tau(1-\frac{\eta_t}{2}(1+\rho n^{-1})) \lambda_{\min}(\mpsi)+\frac{\eta_t}{2} \beta n^{-1} \lambda_{\min}(\mpsi)$. Since $\frac{\eta_t}{2}(1+\rho n^{-1})<1$, the smallest eigenvalue of $\ms_j$ can not become smaller than 
\begin{equation*}
\lambda_{\min}(\mpsi)\frac{\beta}{n+\rho}.
\end{equation*}
Now, assume $\lambda$ and $\lambda'$ to be the largest eigenvalue of $\ms_j$ before and after the update given in~\eqref{eq.updategmm2}. Furthermore, assume  the largest eigenvalue of $\ms_j$ before update be $\norm{\ms_j} = \tau \norm{\mpsi}$. From the update rule~\eqref{eq.updategmm2} and knowing  that the largest eigenvalue of sum of two matrices with positive eigenvalues is not larger than sum of largest eigenvalues of two matrices, we have
\begin{equation*}
\lambda' \leq \lambda + \frac{\eta_t}{2}\norm{\mpsi} \biggl (  -\tau(w+\rho n^{-1}) + w \frac{\norm{\vy_i}}{\norm{\mpsi}}+\beta n^{-1} \biggr).
\end{equation*}
If 
\begin{equation*}
\tau >\max_{w\in[0,1]} \frac{w \frac{\max_i\{\norm{\vy_i}\}}{\norm{\mpsi}}+\beta n^{-1}}{w+\rho n^{-1}},
\end{equation*}
then $\lambda' < \lambda$.
Therefore, the largest eigenvalue of $\ms_j$ remains smaller than 
 \begin{equation*}
\max_{w\in[0,1]} \frac{w n \max_i\{\norm{\vy_i}\}+\beta \norm{\mpsi}}{w n+\rho}.
\end{equation*}
Till now, we have shown that the $\ms_j$s remain in a compact set. We use the same procedure to show that $\omega_j$s also remain in a bounded interval. The Euclidean gradient of the objective with respect to $\omega_j$ for a single data-point is given by:
\begin{equation*}
\nabla_E f_i(\omega_j) = w - \alpha_j+ \frac{\zeta}{n} - \frac{K \zeta}{n} \alpha_j.
\end{equation*}
If $\alpha_j<\frac{\zeta n^{-1}}{1+K\zeta n^{-1}}$, then the gradient is positive and $\omega_j$ is increased after update. From~\eqref{eq.wReparam}, it is clear that $\log(\alpha_j) \leq \omega_j$. Using the update formula $\omega_j^{\text{new}}=\omega_j + \eta_t \nabla_E f_i(\omega_j)$, we get the following lower bound:
\begin{equation*}
\begin{split}
\omega_j^{\text{new}} &\geq \min_{\omega_j\geq \log\bigl(\frac{\zeta n^{-1}}{1+K\zeta n^{-1}}\bigr)} \biggl[ \omega_j+\eta_t \biggl(w - \alpha_j+ \frac{\zeta}{n} - \frac{K \zeta}{n} \alpha_j\biggr) \biggr]\\
&\geq \min_{\omega_j\geq \log\bigl(\frac{\zeta n^{-1}}{1+K\zeta n^{-1}}\bigr)} \biggl[ \omega_j+\eta_t \biggl(1 - \exp(\omega_j)+ \frac{\zeta}{n} - \frac{K \zeta}{n} \exp(\omega_j)\biggr) \biggr]\\
&=\log\biggl(\frac{\zeta n^{-1}}{1+K\zeta n^{-1}}\biggr).
\end{split}
\end{equation*}

From the definition~\eqref{eq.wReparam}, we have $\log(\alpha_i) = \omega_i - \log(\sum_{k=1}^K \exp(\omega_k))$. Using Jensen inequality, we obtain $ \log(\alpha_i) \leq -\sum_{\substack{k=1\\k\neq i}}^{n} \omega_k$. Therefore, we obtain the following upper bound for $\omega_j$
\begin{equation*}
\begin{split}
\omega_j &\leq \log(\alpha_i) -\sum_{\substack{k=1 \\ k\neq i,k\neq j}}^{n} \omega_k \\
&\leq -(K-2) \log\biggl(\frac{\zeta n^{-1}}{1+K\zeta n^{-1}}\biggr).
\end{split}
\end{equation*}
Therefore, one sees that all the parameters ($\ms_j$s and $\omega_j$s) remain in a bounded set.
\end{proof}

\noindent Since the parameters remain bounded, we may invoke the following theorem:
\begin{theorem}[Boumal et al.~\citep{Boumal2016}]
\label{thm:boumal}
Let $\mathcal{M}$ be a compact Riemannian submanifold of a Euclidean space. Let $\retr$ be a retraction on M. If $f$ has a Euclidean Lipschitz continuous gradient in the convex hull of $\mathcal{M}$, then the function satisfies the Lipschitz growth bound with some constant $L$ for all retractions.
\end{theorem}

We have shown above that the iterations of SGD for penalized log-likelihood stay within a compact set. It is also easy to see that the objective has a Euclidean Lipschitz continuous gradient on this set. Therefore, we can invoke Theorem~\ref{thm:boumal} to show that the objective function satisfies condition (i) needed by Theorems~\ref{thm.sgdcond23} and~\ref{thm.sgdcond4}. Furthermore, the objective function has a G-bounded gradient in this compact set and the iterations stay within it. Therefore, condition (iv) needed for Theorem~\ref{thm.sgdcond4} also holds. We summarize this result in the following corollary. 

\begin{corollary}
Assume SGD is used for optimizing the penalized log-likelihood of \gmm, which is given by
\begin{equation*}
f(\{\ms_j \succ 0\}_{j=1}^K,\{\eta_j\}_{j=1}^{K-1}) = \frac{1}{n}\sum_{i=1}^n f_i(\{\ms_j \succ 0\}_{j=1}^K,\{\eta_j\}_{j=1}^{K-1}),
\end{equation*}
where $f_i$ is as in~\eqref{eq.objPenSGD}. Then, the gradient of the objective after $T$ iterations with constant step-size equal to $\eta_t = c / \sqrt{T}$ satisfies 
\begin{equation*}
 \min_{1\le t\le T}\E[\norm{\nabla f^t(\{\ms_j \succ 0\}_{j=1}^K,\{\eta_j\}_{j=1}^{K-1})}^2] \le \frac{1}{\sqrt{T}}\left( \frac{f^*-f^0}{c} + \frac{Lc}{2}G^2\right) = \Oc\left(\tfrac{1}{\sqrt{T}}\right),
\end{equation*} 
where $f^t$ is the penalized objective evaluated at the value of parameters after $t$ iterations; $f^*$ is the value of penalized objective at its optimum; $f^0$ is the value of the objective at its initial point; $L$ is the Lipschitz-growth bound constant; and $G$ is the constant for the G-bounded condition of the gradient.
  \end{corollary}

\section{Experiments}
\label{sec:expt}

\begin{figure}[htbp]
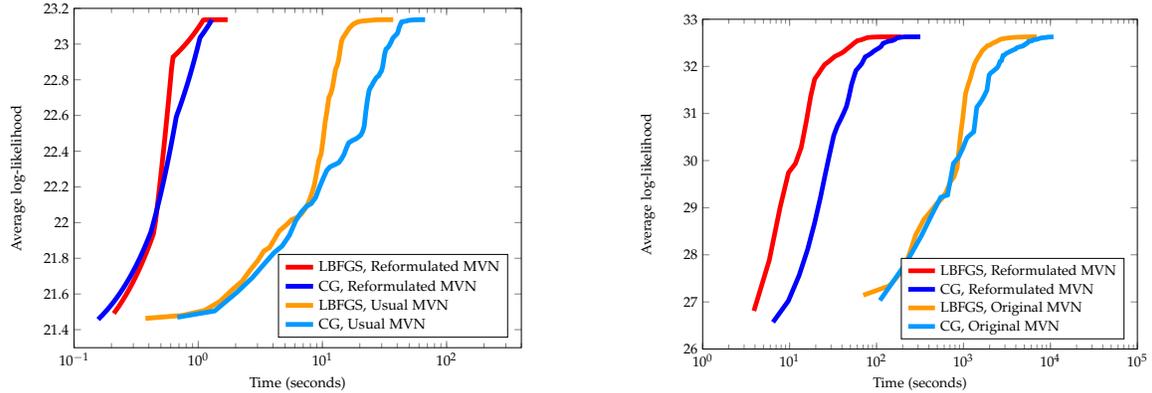
\small
\centering
  \label{fig:singlecom}%
    \resizebox{.45\textwidth}{!}{\input{./naturals1}}
\hfill%
  \label{fig:sevencom}%
      \resizebox{.45\textwidth}{!}{\input{./naturals2}}
\caption{\small
\label{fig:reparam} The effect of reformulation in convergence speed of manifold CG and manifold LBFGS methods ($d=35$); note that the X-axis (time) is on a logarithmic scale~\cite{hosseini2015matrix}.}
\end{figure}

\begin{figure}[htbp]
\centering
\subfigure{%
  \label{fig:magic}%
  \resizebox{.45\textwidth}{!}{
%
%
\definecolor{mycolor1}{rgb}{0.00000,0.00000,0.17241}%
\definecolor{mycolor2}{rgb}{1.00000,0.10345,0.72414}%
\definecolor{mycolor3}{rgb}{1.00000,0.82759,0.00000}%
\begin{tikzpicture}

\begin{axis}[%
scale =0.55,
width=6.028in,
height=4.754in,
at={(1.011in,0.642in)},
scale only axis,
xmin=0,
xmax=50,
tick align=outside,
xlabel={Iterations},
ymode=log,
ymin=1e-05,
ymax=10,
ylabel={Averaged Cost Difference},
axis background/.style={fill=white},
legend style={legend cell align=left,align=left,draw=white!15!black}
]
\addplot [color=blue,solid,line width=3.0pt]
  table[row sep=crcr]{%
0	5.59009639433904\\
1	0.979512413759981\\
2	0.199861763855825\\
3	0.0394643012073388\\
4	0.00769695966669204\\
};
\addlegendentry{SGD (it=5)};

\addplot [color=red,solid,line width=3.0pt]
  table[row sep=crcr]{%
0	5.59009639433904\\
1	1.43311924478897\\
2	0.951236432144562\\
3	0.59990103543727\\
4	0.357672294334517\\
5	0.230228258655202\\
6	0.146538619768982\\
7	0.0956466519851915\\
8	0.0611274638392523\\
9	0.0372359299427458\\
10	0.0252237573659002\\
11	0.0152047549707603\\
12	0.0090179056019295\\
13	0.00479196147207972\\
14	0.00216709744180932\\
15	0.00089091863047841\\
16	0.000370836051725831\\
17	0.000152962044552396\\
18	8.48848603105523e-05\\
19	6.03594112682515e-05\\
};
\addlegendentry{SGD (it=20)};

\addplot [color=green,solid,line width=3.0pt]
  table[row sep=crcr]{%
0	5.59009639433904\\
1	2.02854227830542\\
2	1.70847459422006\\
3	1.45564784643182\\
4	1.19184484331264\\
5	0.923024408651607\\
6	0.751715380950859\\
7	0.626522871425724\\
8	0.533290436844265\\
9	0.44445696731211\\
10	0.356377730674641\\
11	0.297376543128298\\
12	0.252061309901393\\
13	0.218418006312859\\
14	0.169280892073367\\
15	0.148023175745024\\
16	0.123670294589701\\
17	0.106623316400569\\
18	0.0873960199081196\\
19	0.0718044466917149\\
20	0.0588709952595536\\
21	0.0504799446392425\\
22	0.0416412317270911\\
23	0.0361587646137451\\
24	0.0290529566834685\\
25	0.0241373709533441\\
26	0.0212449792398139\\
27	0.0157900113130154\\
28	0.0133731852732204\\
29	0.0111885111410075\\
30	0.00856945096838047\\
31	0.00672053907004511\\
32	0.00542801176154484\\
33	0.00360145404208012\\
34	0.00260583411465731\\
35	0.0020787886562097\\
36	0.00134962580332143\\
37	0.000900450098370698\\
38	0.000683069833414152\\
39	0.000466259348357312\\
40	0.000312688340770251\\
41	0.000187418323392308\\
42	0.000136047336440015\\
43	8.5400481797393e-05\\
44	5.62011728320044e-05\\
45	3.75552405600388e-05\\
46	2.59082245115394e-05\\
47	1.87854155271339e-05\\
48	1.39932668012932e-05\\
49	1.04515657568527e-05\\
};
\addlegendentry{SGD (it=50)};

\addplot [color=mycolor1,solid,line width=3.0pt]
  table[row sep=crcr]{%
0	5.59009631933904\\
1	2.06604162812699\\
2	1.28490673110454\\
3	0.863174168356192\\
4	0.625431932163398\\
5	0.470443559119612\\
6	0.360516910882733\\
7	0.278572374638244\\
8	0.215966829586229\\
9	0.16733128295888\\
10	0.129469928119939\\
11	0.100083336968105\\
12	0.0772671738677886\\
13	0.0596154718877813\\
14	0.0460128695583073\\
15	0.0355320226380087\\
16	0.0274469493217246\\
17	0.0212089347896338\\
18	0.0163950569368367\\
19	0.012675762599109\\
20	0.00979754734670024\\
21	0.00756810993990698\\
22	0.00584157403126184\\
23	0.00450588616951109\\
24	0.00347385190613281\\
25	0.00267725681935005\\
26	0.00206280888894028\\
27	0.00158904292450757\\
28	0.00122382125514875\\
29	0.000942302619808544\\
30	0.000725314888562423\\
31	0.000558072013848943\\
32	0.000429173899767932\\
33	0.000329831891026089\\
34	0.000253270827673902\\
35	0.000194267926612923\\
36	0.000148797267371492\\
37	0.000113755682082939\\
38	8.67513854991842e-05\\
39	6.59409904919528e-05\\
40	4.99038270618257e-05\\
41	3.75450464034088e-05\\
42	2.80209155860689e-05\\
43	2.0681233664277e-05\\
44	1.50249441901451e-05\\
45	1.0665925152864e-05\\
46	7.30662552328454e-06\\
47	4.71774893995303e-06\\
48	2.72259381972617e-06\\
49	1.1849896459637e-06\\
50	0\\
};
\addlegendentry{EM};

\addplot [color=mycolor2,solid,line width=3.0pt]
  table[row sep=crcr]{%
0	5.59009639433904\\
1	4.62201753984071\\
2	2.89122225760966\\
3	1.35862603576276\\
4	0.508827754566934\\
6	0.300869787800387\\
8	0.145791280363142\\
10	0.0577494764878139\\
12	0.0197422591135457\\
14	0.0105164739645964\\
16	0.00557263470237146\\
17	0.00226059372340615\\
18	0.001392599823415\\
20	0.000703477098237926\\
22	0.000263741669645157\\
24	0.000110650150823233\\
26	6.07478185266075e-05\\
28	3.10306809225835e-05\\
29	8.80395938906986e-06\\
};
\addlegendentry{LBFGS};

\addplot [color=mycolor3,solid,line width=3.0pt]
  table[row sep=crcr]{%
0	5.59009639433904\\
1	4.62201753984071\\
2	3.36032605928573\\
3	2.13337989978533\\
4	1.32612481490831\\
5.5	0.612935670352677\\
6.5	0.385828465593725\\
7.5	0.216853491314179\\
9.5	0.0984349756338396\\
10.5	0.0376775776725431\\
13	0.0216148071170785\\
14	0.011731137616195\\
15.5	0.00758350087438586\\
16.5	0.00443191669248577\\
17.5	0.00215408982917609\\
19.5	0.000569333785882975\\
21	0.000143551042135925\\
22.5	6.83851543463732e-05\\
23.5	5.15173795818669e-05\\
24.5	3.64117708429035e-05\\
25.5	1.83127160653385e-05\\
27.5	4.2918373068801e-06\\
};
\addlegendentry{CG};

\end{axis}
\end{tikzpicture}%
}
  }%
\hfill%
\subfigure{%
  \label{fig:year}%
  \resizebox{.45\textwidth}{!}{
%
%
\definecolor{mycolor1}{rgb}{0.00000,0.00000,0.17241}%
\definecolor{mycolor2}{rgb}{1.00000,0.10345,0.72414}%
\definecolor{mycolor3}{rgb}{1.00000,0.82759,0.00000}%
\begin{tikzpicture}

\begin{axis}[%
scale = 0.55,
width=6.028in,
height=4.754in,
at={(1.011in,0.642in)},
scale only axis,
xmin=0,
xmax=90,
tick align=outside,
xlabel={Iterations},
ymode=log,
ymin=1e-05,
ymax=10,
ylabel={Averaged Cost Difference},
axis background/.style={fill=white},
legend style={legend cell align=left,align=left,draw=white!15!black}
]
\addplot [color=blue,solid,line width=3.0pt]
  table[row sep=crcr]{%
0	5.67590342424566\\
1	1.30994979979677\\
2	0.451191052574032\\
3	0.282902605287731\\
4	0.24625346737696\\
};
\addlegendentry{SGD (it=5)};

\addplot [color=red,solid,line width=3.0pt]
  table[row sep=crcr]{%
0	5.67590342424566\\
1	1.56857975145485\\
2	1.08397212561781\\
3	0.724470716130885\\
4	0.511321857872723\\
5	0.266481998777337\\
6	0.133841077107874\\
7	0.0865767268990538\\
8	0.0580067203711252\\
9	0.0391994964243665\\
10	0.0259577101805348\\
11	0.0185678889924077\\
12	0.014331140758955\\
13	0.0122475954987635\\
14	0.0113850071830868\\
15	0.0110214958660677\\
16	0.0108632710683452\\
17	0.0107868939680884\\
18	0.0107420256349968\\
19	0.0107139819069459\\
};
\addlegendentry{SGD (it=20)};

\addplot [color=green,solid,line width=3.0pt]
  table[row sep=crcr]{%
0	5.67590342424566\\
1	2.15204781243785\\
2	1.76807186882611\\
3	1.51285256574542\\
4	1.23141793682789\\
5	1.03313895412003\\
6	0.898787821666133\\
7	0.726400848356509\\
8	0.519103035198256\\
9	0.382933119988849\\
10	0.306421912992015\\
11	0.253418415623635\\
12	0.217039163798241\\
13	0.178825961658625\\
14	0.146502843818801\\
15	0.128948624784002\\
16	0.103172862302699\\
17	0.0884223638381272\\
18	0.0769369104602191\\
19	0.0630654190107762\\
20	0.0515802641191172\\
21	0.0449968613634653\\
22	0.038099049611489\\
23	0.0312945978358528\\
24	0.0264539679150957\\
25	0.0226857543378145\\
26	0.0191334687485494\\
27	0.0162520315184551\\
28	0.0142263697301743\\
29	0.0124596079975277\\
30	0.0111612448809879\\
31	0.0101705061388842\\
32	0.00943236950818971\\
33	0.00897349411305726\\
34	0.00857164564025936\\
35	0.00829607402388177\\
36	0.00811779576739724\\
37	0.00797362853493411\\
38	0.00787762889527954\\
39	0.00779947778421786\\
40	0.00774413842027855\\
41	0.00770154066809425\\
42	0.00766786618556381\\
43	0.00763953552490904\\
44	0.00761660894906413\\
45	0.00759796684115344\\
46	0.00758267991989925\\
47	0.00756947117599793\\
48	0.00755847546388821\\
49	0.00754916442951981\\
};
\addlegendentry{SGD (it=50)};

\addplot [color=mycolor1,solid,line width=3.0pt]
  table[row sep=crcr]{%
0	5.67590324924568\\
1	2.41755276047445\\
2	1.74626801088705\\
3	1.36542068740102\\
4	1.10917286739067\\
5	0.919833880321857\\
6	0.773337702610561\\
7	0.656624151609691\\
8	0.560401187028646\\
9	0.479008246257173\\
10	0.410422868277784\\
11	0.353452098730187\\
12	0.305882862044086\\
13	0.265611033847748\\
14	0.231069639599603\\
15	0.201250416018333\\
16	0.175366043029697\\
17	0.153055348615325\\
18	0.13412724616196\\
19	0.118079887187058\\
20	0.104335377339957\\
21	0.0924381852580183\\
22	0.0820066038422311\\
23	0.0727284180492731\\
24	0.064376436693486\\
25	0.0568209726184321\\
26	0.0500030145430372\\
27	0.0438723133388095\\
28	0.0383791332693235\\
29	0.0334897439895201\\
30	0.0291747231443651\\
31	0.0253944307494365\\
32	0.0220973174112871\\
33	0.0192292924810076\\
34	0.0167385302258509\\
35	0.0145764155889196\\
36	0.0126981880675103\\
37	0.0110641285733095\\
38	0.00964031623774986\\
39	0.00839829488693056\\
40	0.00731412635006023\\
41	0.00636749316628027\\
42	0.00554102463024719\\
43	0.00481975082067976\\
44	0.00419064605128483\\
45	0.00364228479538298\\
46	0.00316460623460557\\
47	0.00274874862949304\\
48	0.00238691633360588\\
49	0.00207226128821247\\
50	0.00179877413014395\\
51	0.00156118486482626\\
52	0.0013548738396949\\
53	0.00117579327516637\\
54	0.00102039882761318\\
55	0.000885590013055548\\
56	0.000768658033557301\\
57	0.000667239689342125\\
58	0.000579276350137548\\
59	0.000502977280817163\\
60	0.000436786838086789\\
61	0.00037935517185872\\
62	0.000329512132850596\\
63	0.000286244108622213\\
64	0.000248673565167223\\
65	0.000216041025218772\\
66	0.000187689288587478\\
67	0.000163049654062775\\
68	0.000141629933871457\\
69	0.000123004063709686\\
70	0.00010680310964517\\
71	9.27074995757948e-05\\
72	8.04403231882134e-05\\
73	6.97615477136537e-05\\
74	6.04630383520544e-05\\
75	5.23642530581014e-05\\
76	4.53085388016916e-05\\
77	3.91599267786091e-05\\
78	3.38003694366762e-05\\
79	2.91273480499399e-05\\
80	2.50518111997167e-05\\
81	2.14963822386949e-05\\
82	1.83938159494801e-05\\
83	1.56856564146324e-05\\
84	1.33210798054506e-05\\
85	1.12558838978316e-05\\
86	9.45162083354489e-06\\
87	7.87483762820784e-06\\
88	6.49642272776418e-06\\
89	5.29103679269838e-06\\
90	4.2366171655317e-06\\
};
\addlegendentry{EM};

\addplot [color=mycolor2,solid,line width=3.0pt]
  table[row sep=crcr]{%
0	5.67590342424566\\
1	4.90777919336633\\
2	3.61726471359609\\
3	2.72581263236706\\
4	1.88786421558245\\
5	1.18224073789006\\
7	0.929357220298606\\
8	0.698793149116121\\
10	0.552675399330575\\
11	0.428270727946376\\
12	0.394891508852389\\
13	0.349782795203474\\
14	0.294927425248858\\
15	0.28070597743536\\
16	0.261169021332655\\
17	0.223974431340906\\
18	0.177538491998675\\
20	0.151357504023046\\
22	0.138771560372319\\
23	0.12861386081039\\
24	0.115621518894173\\
25	0.100009082908393\\
26	0.0881280696690467\\
27	0.0789831171441122\\
28	0.0643987310318437\\
29	0.0397794991948572\\
30	0.0283031665951157\\
31	0.0150012251420364\\
32	0.0130588085205829\\
33	0.0104808607498299\\
34	0.00625843883842947\\
35	0.00525288758169751\\
36	0.00321962587641167\\
37	0.00215265083514282\\
38	0.00137313073710743\\
39	0.00103175616258966\\
40	0.000633999887369896\\
42	0.000442706753013056\\
44	0.000323941361699553\\
45	0.000164639223363849\\
46	8.88620142802665e-05\\
47	7.94996976267726e-05\\
48	5.61465326782695e-05\\
49	2.96137648589934e-05\\
51	1.75831462669862e-05\\
52	7.93231884443912e-06\\
53	2.72187083538711e-06\\
55	0\\
};
\addlegendentry{LBFGS};

\addplot [color=mycolor3,solid,line width=3.0pt]
  table[row sep=crcr]{%
0	5.67590342424566\\
1	4.90777919336633\\
2	4.00199142721068\\
3	3.17284175129977\\
4	2.58398731845784\\
5	1.47084461740539\\
7	1.04308644963817\\
8.5	0.86654355531326\\
9.5	0.720369600640776\\
11	0.586108866196874\\
12	0.511734814846982\\
13	0.431747675555528\\
15	0.35854031223635\\
16.5	0.314325985761108\\
17.5	0.296979884380121\\
18.5	0.283400582613723\\
20.5	0.268797136553992\\
21.5	0.254135305269571\\
22.5	0.216218101503408\\
25	0.193916784266612\\
26.5	0.171344221677288\\
27.5	0.162187837870796\\
28.5	0.150776423446779\\
29.5	0.136470013194796\\
31.5	0.127553572023757\\
33	0.108950201706598\\
34	0.102272495617221\\
35	0.0837739814070062\\
37.5	0.0776881280045316\\
38.5	0.0679721318032875\\
39.5	0.0465487614657931\\
42	0.0361608401916413\\
43	0.0277133027909002\\
44.5	0.0192489912796106\\
46	0.0145490837341384\\
47	0.012516190030027\\
48	0.0104959502522988\\
49	0.00788244021298112\\
51	0.0039295897621443\\
53	0.00257980106616884\\
54	0.000778549361378111\\
56	0.000612728404050245\\
57	0.000446066783190702\\
58.5	0.000345493471037628\\
59.5	0.000194025127335351\\
61	0.000102561381865485\\
62	2.50473372531701e-05\\
64	1.3242388824608e-05\\
65	7.09513712138232e-06\\
67	1.72380382679194e-06\\
};
\addlegendentry{CG};

\end{axis}
\end{tikzpicture}%
}
  }
\caption{\small
\label{fig:iter1}\small Comparison of optimization methods on natural image data ($d=35$, $n= 200000$). Y-axis:
best cost minus current cost values. X-axis: number of function and gradient evaluations. Right: 3 number of components. Left: 7 number of components.}
\end{figure}
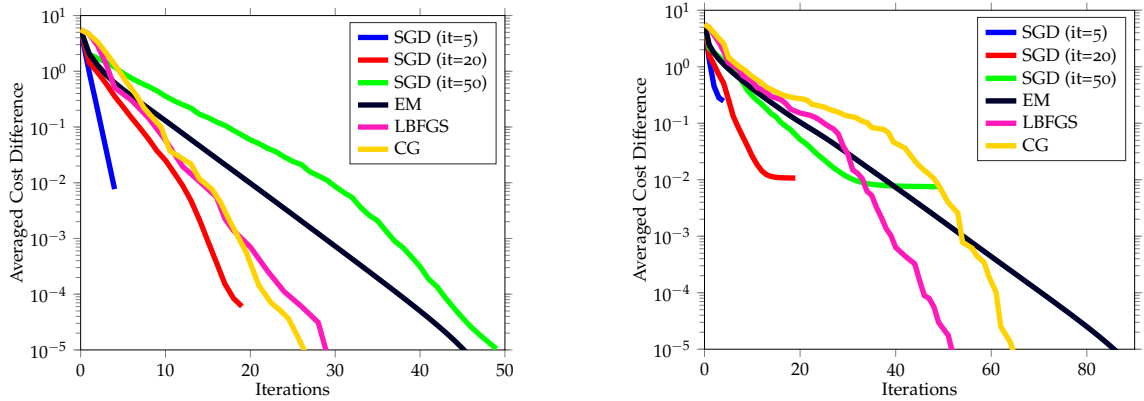

In all experiments, the parameters of the penalizer in~\eqref{eq.pen1} are $\rho=\kappa=0.01$ and $\alpha=\beta=1$. The parameter $\mlambda$ is set to 0.01 of sample covariance of the data and $\vlambda$ is sample mean of the data. The parameter $\zeta$ of the penalizer in~\eqref{eq.priorweight} is set to 1. We initialize the mixture parameters using \texttt{k-means++}~\citep{arthur2007} by testing 30 different initial candidate and choosing the one with the best cost function. All methods stop when the difference between cost functions falls below $10^{-6}$.

In order to show the efficacy of SGD, we fix the step-size rule in all experiments. We use exponential decay for the step-size. Given the maximum number of epochs, we set the starting step-size to 1 and the last step-size to $10^{-3}$. The batch size is set to be equal to the dimensionality of data.

For the deterministic Riemannian optimization methods, we use exponential map and parallel transport as they lead to superior performance compared to other kinds of retractions and vector transports. For Riemannian SGD, we report the result of using Euclidean retraction. We also tested a more expensive exponential map and a different positivity-preserving retraction~\cite{jeuris2012survey}. However, we observed no difference in cost function decrease as a function of gradient evaluations. 

In the first experiment, the effect of the problem reformulation of Section~\ref{sec:prob} is investigate. This effect is shown if Figure~\ref{fig:reparam}. The left plot is the result of optimization for a single Gaussian and the right plot is the result for \gmm with seven components. It can be seen that the reformulation has significant effect on the convergence speed. 

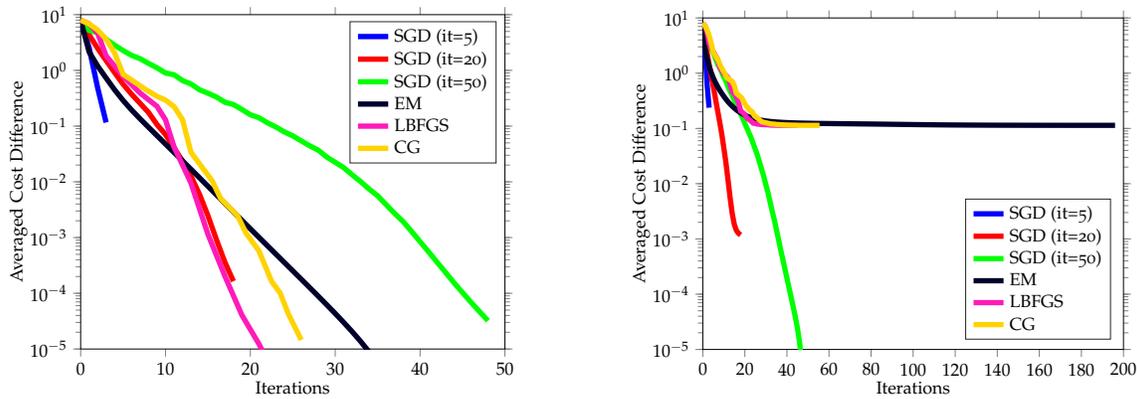
\begin{figure}[htbp]
\centering
\subfigure{%
  \label{fig:magic2}%
  \resizebox{.45\textwidth}{!}{
%
%
\definecolor{mycolor1}{rgb}{0.00000,0.00000,0.17241}%
\definecolor{mycolor2}{rgb}{1.00000,0.10345,0.72414}%
\definecolor{mycolor3}{rgb}{1.00000,0.82759,0.00000}%
\begin{tikzpicture}

\begin{axis}[%
scale =0.55,
width=6.028in,
height=4.754in,
at={(1.011in,0.642in)},
scale only axis,
xmin=0,
xmax=50,
tick align=outside,
xlabel={Iterations},
ymode=log,
ymin=1e-05,
ymax=10,
ylabel={Averaged Cost Difference},
axis background/.style={fill=white},
legend style={legend cell align=left,align=left,draw=white!15!black}
]
\addplot [color=blue,solid,line width=3.0pt]
  table[row sep=crcr]{%
0	7.99660646609543\\
1	2.5344984981715\\
2	0.489321471896673\\
3	0.116262853687829\\
};
\addlegendentry{SGD (it=5)};

\addplot [color=red,solid,line width=3.0pt]
  table[row sep=crcr]{%
0	7.99660646609543\\
1	4.33125672026279\\
2	2.48977394115408\\
3	1.50611558007803\\
4	0.915017163294564\\
5	0.565707386139678\\
6	0.374522953197499\\
7	0.256679599637621\\
8	0.175833837789696\\
9	0.106529097488277\\
10	0.0707069477492013\\
11	0.0408732701929893\\
12	0.0236958491023387\\
13	0.0126540738255869\\
14	0.00598257889475917\\
15	0.00246764943840105\\
16	0.000914295305022961\\
17	0.000360330051492497\\
18	0.000164322662413952\\
};
\addlegendentry{SGD (it=20)};

\addplot [color=green,solid,line width=3.0pt]
  table[row sep=crcr]{%
0	7.99660646609543\\
1	5.22577341027606\\
2	4.73839474055795\\
3	3.62266731381571\\
4	2.83897016837935\\
5	2.27248657836233\\
6	1.8607978293434\\
7	1.62196126282369\\
8	1.33526857009785\\
9	1.1104455234388\\
10	0.896939322197866\\
11	0.835329865497528\\
12	0.657204587113114\\
13	0.570553816825466\\
14	0.448670777314732\\
15	0.388906709392344\\
16	0.332866200732369\\
17	0.267666852930176\\
18	0.2440330290394\\
19	0.198692838945547\\
20	0.160057170260416\\
21	0.141533934962951\\
22	0.11495273378047\\
23	0.0965330938175128\\
24	0.0786741496943648\\
25	0.0668839229178246\\
26	0.0542776706162016\\
27	0.0441839906316233\\
28	0.0370844139717832\\
29	0.0278706195144451\\
30	0.0224093211634653\\
31	0.0180027078472165\\
32	0.0133349574810069\\
33	0.0100951948932391\\
34	0.0074044860054272\\
35	0.00561457127170684\\
36	0.00389469431269163\\
37	0.00271359717321218\\
38	0.00193553245617295\\
39	0.00128738759920566\\
40	0.000843115314438592\\
41	0.000553999744838052\\
42	0.00036130020752978\\
43	0.000236964013410557\\
44	0.00015397971624509\\
45	0.000102987180227387\\
46	6.90441485744486e-05\\
47	4.69703066983129e-05\\
48	3.21423924702913e-05\\
};
\addlegendentry{SGD (it=50)};

\addplot [color=mycolor1,solid,line width=3.0pt]
  table[row sep=crcr]{%
0	7.99660643698873\\
1	2.14005797756825\\
2	1.20181435634899\\
3	0.729219213532176\\
4	0.449266672520011\\
5	0.289520015237102\\
6	0.195045988738869\\
7	0.135440491910785\\
8	0.0953697871697443\\
9	0.0673850730378973\\
10	0.0476518364830412\\
11	0.0337079728285374\\
12	0.0238403384447423\\
13	0.0168289711465093\\
14	0.0118504042581691\\
15	0.00833051985615896\\
16	0.00585296174917005\\
17	0.00411367803934581\\
18	0.00289370153497259\\
19	0.00203774464418416\\
20	0.00143664385076647\\
21	0.00101396501010242\\
22	0.000716296128629779\\
23	0.000506336585715417\\
24	0.000358021820346721\\
25	0.00025311146028173\\
26	0.000178815426451706\\
27	0.000126146528756976\\
28	8.87769800925753e-05\\
29	6.22432755505997e-05\\
30	4.33918753515172e-05\\
31	2.99916473238682e-05\\
32	2.04622095623108e-05\\
33	1.3683009129295e-05\\
34	8.85886758794641e-06\\
35	5.42509361878274e-06\\
36	2.98045451074813e-06\\
37	1.23970905008264e-06\\
38	0\\
};
\addlegendentry{EM};

\addplot [color=mycolor2,solid,line width=3.0pt]
  table[row sep=crcr]{%
0	7.99660646609611\\
1	6.82389963134017\\
2	4.66655883702066\\
3	1.82185221788544\\
5	0.786754275466677\\
7	0.392386929575252\\
9	0.226141853066295\\
10	0.129751477906055\\
11	0.0424911809194413\\
13	0.00952514878103727\\
15	0.00114446156645442\\
17	0.00020256588706502\\
19	4.06120538940513e-05\\
21	1.2398464136254e-05\\
23	3.08457835274112e-06\\
};
\addlegendentry{LBFGS};

\addplot [color=mycolor3,solid,line width=3.0pt]
  table[row sep=crcr]{%
0	7.99660646609611\\
1	6.82389963134017\\
2	5.31732130708488\\
3	3.64021545531853\\
4	2.06401922937819\\
5	0.852740557796309\\
7.5	0.470160304737874\\
9	0.342208297919527\\
10	0.292469093462699\\
11	0.224592034204264\\
12	0.131073337551101\\
13	0.0338481098287957\\
15.5	0.010412361217476\\
16.5	0.00501399142805781\\
18.5	0.00244638875737735\\
19.5	0.0011901694635128\\
21	0.000580023922708506\\
22.5	0.000163891365119184\\
23.5	0.000101639199542092\\
24.5	4.19071457002929e-05\\
26	1.45626244005825e-05\\
};
\addlegendentry{CG};

\end{axis}
\end{tikzpicture}%
}
  }%
\hfill%
\subfigure{%
  \label{fig:year2}%
  \resizebox{.45\textwidth}{!}{
%
%
\definecolor{mycolor1}{rgb}{0.00000,0.00000,0.17241}%
\definecolor{mycolor2}{rgb}{1.00000,0.10345,0.72414}%
\definecolor{mycolor3}{rgb}{1.00000,0.82759,0.00000}%
\begin{tikzpicture}

\begin{axis}[%
scale =0.55,
width=6.028in,
height=4.754in,
at={(1.011in,0.642in)},
scale only axis,
xmin=0,
xmax=200,
tick align=outside,
xlabel={Iterations},
ymode=log,
ymin=1e-05,
ymax=10,
ylabel={Averaged Cost Difference},
axis background/.style={fill=white},
legend style={legend cell align=left,align=left,draw=white!15!black},
legend pos=south east
]
\addplot [color=blue,solid,line width=3.0pt]
  table[row sep=crcr]{%
0	8.10334394013391\\
1	2.91994144976438\\
2	0.657044989526028\\
3	0.237402516959719\\
};
\addlegendentry{SGD (it=5)};

\addplot [color=red,solid,line width=3.0pt]
  table[row sep=crcr]{%
0	8.10334394013391\\
1	4.03139310336157\\
2	2.39385216481914\\
3	1.44220474195274\\
4	0.895660413575399\\
5	0.57228691349404\\
6	0.35097273582106\\
7	0.215347937370652\\
8	0.132953861101683\\
9	0.0803881946581697\\
10	0.0444815165729437\\
11	0.0240764296526521\\
12	0.011656788900666\\
13	0.00552678824374198\\
14	0.00287657418775211\\
15	0.00181742204819102\\
16	0.0014232692311964\\
17	0.0012588124534858\\
18	0.0011818831192727\\
};
\addlegendentry{SGD (it=20)};

\addplot [color=green,solid,line width=3.0pt]
  table[row sep=crcr]{%
0	8.10334394013391\\
1	5.330421381958\\
2	4.50163944123021\\
3	3.55853054147697\\
4	2.79996430743903\\
5	2.34005840649014\\
6	1.85940531045887\\
7	1.48870110882358\\
8	1.26249173973051\\
9	1.04376250071078\\
10	0.850750512052386\\
11	0.713534480189622\\
12	0.569799463935141\\
13	0.469348671368991\\
14	0.395044682595454\\
15	0.325788255302072\\
16	0.267802166840667\\
17	0.221942716571476\\
18	0.18510282545887\\
19	0.151333611862114\\
20	0.123052096962354\\
21	0.100996989421134\\
22	0.0820672613402991\\
23	0.0673266086502053\\
24	0.0526087637118309\\
25	0.0418415836409167\\
26	0.0324676276107851\\
27	0.0243141765252233\\
28	0.0181499488325727\\
29	0.0134727550103548\\
30	0.00951819170926171\\
31	0.00689820066314439\\
32	0.00470235707621214\\
33	0.00318727282272846\\
34	0.00221739038369151\\
35	0.00147679223321973\\
36	0.000982668145539378\\
37	0.000645515849498679\\
38	0.000433919342810896\\
39	0.000290037981770297\\
40	0.000195688362964574\\
41	0.000130555759753292\\
42	8.70959962213647e-05\\
43	5.86030898830359e-05\\
44	3.86498549573844e-05\\
45	2.42752767860566e-05\\
46	1.35866138606389e-05\\
47	6.01555478851878e-06\\
48	0\\
};
\addlegendentry{SGD (it=50)};

\addplot [color=mycolor1,solid,line width=3.0pt]
  table[row sep=crcr]{%
0	8.10334387221825\\
1	2.83783526191445\\
2	1.86498117682471\\
3	1.34300399169919\\
4	1.02690456823616\\
5	0.821936934357254\\
6	0.680288387920029\\
7	0.576084291247554\\
8	0.495981274564102\\
9	0.432648254522405\\
10	0.381303670792761\\
11	0.338972497429879\\
12	0.304292837391699\\
13	0.275903935148904\\
14	0.252517409660641\\
15	0.233082957008904\\
16	0.216853054667965\\
17	0.203293481472912\\
18	0.191940446243521\\
19	0.18237914814766\\
20	0.174310196159297\\
21	0.167497480279678\\
22	0.161719990926471\\
23	0.156772786355219\\
24	0.152513746390824\\
25	0.148844861908017\\
26	0.145685483152022\\
27	0.142963375902426\\
28	0.140613822292707\\
29	0.138582079646866\\
30	0.13682247461594\\
31	0.135294924651106\\
32	0.133963426291402\\
33	0.132796498788466\\
34	0.131767616858838\\
35	0.130854943402781\\
36	0.130041274920956\\
37	0.12931332868267\\
38	0.128660504083868\\
39	0.128073834212785\\
40	0.127545378734737\\
41	0.127067993355226\\
42	0.126635286281783\\
43	0.126241625173463\\
44	0.125882128558956\\
45	0.125552602870499\\
46	0.125249439775118\\
47	0.124969517109818\\
48	0.124710122563329\\
49	0.124468895151665\\
50	0.124243776426496\\
51	0.124032968275344\\
52	0.123834896796211\\
53	0.123648180875989\\
54	0.123471603539713\\
55	0.12330408533812\\
56	0.123144660832722\\
57	0.122992459279516\\
58	0.122846689591483\\
59	0.122706628547718\\
60	0.122571611479529\\
61	0.122441025088932\\
62	0.122314301746329\\
63	0.122190914385101\\
64	0.122070371380772\\
65	0.121952211274646\\
66	0.121835997818565\\
67	0.121721316286454\\
68	0.121607772232494\\
69	0.121494993634613\\
70	0.121382636467473\\
71	0.121270391697266\\
72	0.121157988953982\\
73	0.121045191870834\\
74	0.120931785787306\\
75	0.120817567564799\\
76	0.120702348085345\\
77	0.120585963397303\\
78	0.120468274462326\\
79	0.120349152390148\\
80	0.120228476392967\\
81	0.120106155562539\\
82	0.119982156296274\\
83	0.119856515871916\\
84	0.119729334116812\\
85	0.119600749097813\\
86	0.119470909964974\\
87	0.119339953866358\\
88	0.119207989273619\\
89	0.119075091055237\\
90	0.118941306265697\\
91	0.118806657895149\\
92	0.118671142528854\\
93	0.118534740728407\\
94	0.118397460485966\\
95	0.11825940462235\\
96	0.118120824430662\\
97	0.117982122402893\\
98	0.117843798007243\\
99	0.117706372881102\\
100	0.117570340378236\\
101	0.117436149548098\\
102	0.117304205197371\\
103	0.117174866690647\\
104	0.117048440816717\\
105	0.116925173258139\\
106	0.11680524496218\\
107	0.116688775830198\\
108	0.116575834137691\\
109	0.116466448142617\\
110	0.116360616559959\\
111	0.116258315806036\\
112	0.116159503801228\\
113	0.116064121720854\\
114	0.1159720956543\\
115	0.115883339490907\\
116	0.115797759170391\\
117	0.115715257556872\\
118	0.115635738786779\\
119	0.115559111063305\\
120	0.115485287220153\\
121	0.115414183230115\\
122	0.115345715439659\\
123	0.1152797980195\\
124	0.115216341597517\\
125	0.115155253382184\\
126	0.115096438179975\\
127	0.115039799245807\\
128	0.1149852383965\\
129	0.114932655563408\\
130	0.114881948555436\\
131	0.114833013723057\\
132	0.11478574800077\\
133	0.114740052545741\\
134	0.114695837794471\\
135	0.114653028794379\\
136	0.114611568353837\\
137	0.114571416020055\\
138	0.114532545112567\\
139	0.114494943518622\\
140	0.114458621049522\\
141	0.114423616502833\\
142	0.114389989565424\\
143	0.114357793922338\\
144	0.114327053834131\\
145	0.114297761945451\\
146	0.114269889881399\\
147	0.114243398395956\\
148	0.114218243206679\\
149	0.114194377727145\\
150	0.114171754387797\\
151	0.114150325362615\\
152	0.114130042900371\\
153	0.114110859446512\\
154	0.114092727749949\\
155	0.114075601115033\\
156	0.114059433681163\\
157	0.1140441805628\\
158	0.114029797775785\\
159	0.114016242069525\\
160	0.11400347084151\\
161	0.113991442215493\\
162	0.113980115247614\\
163	0.113969450211499\\
164	0.113959408918774\\
165	0.113949955017958\\
166	0.113941054259371\\
167	0.113932674702525\\
168	0.113924786820021\\
169	0.113917363519036\\
170	0.11391038010089\\
171	0.113903814140599\\
172	0.113897645340018\\
173	0.113891855272939\\
174	0.113886427049067\\
175	0.113881344900278\\
176	0.113876593742106\\
177	0.113872158788638\\
178	0.113868025295488\\
179	0.113864178428685\\
180	0.113860603236333\\
181	0.113857284719572\\
182	0.113854207931297\\
183	0.113851358108533\\
184	0.113848720777682\\
185	0.113846281888044\\
186	0.113844027887978\\
187	0.113841945805675\\
188	0.113840023293889\\
189	0.113838248669396\\
190	0.113836610918199\\
191	0.113835099703493\\
192	0.113833705365622\\
193	0.113832418888421\\
194	0.113831231889776\\
195	0.113830136589883\\
196	0.113829125779532\\
};
\addlegendentry{EM};

\addplot [color=mycolor2,solid,line width=3.0pt]
  table[row sep=crcr]{%
0	8.10334394013371\\
1	7.25251073513803\\
2	5.65292630438709\\
3	3.42927652568486\\
5	2.43832440530544\\
7	1.80633915762054\\
8	1.27304917333106\\
10	0.976055631635433\\
11	0.920594273772565\\
12	0.823526460071704\\
13	0.642744017916698\\
14	0.458257081576505\\
15	0.422223703151147\\
16	0.371083120972941\\
17	0.291612845163954\\
18	0.207412258098984\\
20	0.179195853589718\\
21	0.153587028494741\\
22	0.14550595403793\\
23	0.134708260731742\\
24	0.124617763846388\\
26	0.119940305353907\\
28	0.117634049980353\\
30	0.116766374732521\\
31	0.116018871266945\\
32	0.115371127880252\\
33	0.114664084997614\\
34	0.114433384534522\\
35	0.114142914459777\\
36	0.114021188723981\\
37	0.113947260579593\\
38	0.113921539752937\\
39	0.113889370383376\\
40	0.11386368633616\\
41	0.113839868290455\\
43	0.113828065082117\\
45	0.113822722977176\\
};
\addlegendentry{LBFGS};

\addplot [color=mycolor3,solid,line width=3.0pt]
  table[row sep=crcr]{%
0	8.10334394013371\\
1	7.25251073513803\\
2	6.13558899237086\\
3	4.81903637809238\\
4	3.49229038676378\\
5	2.36622404571091\\
7	1.80884976242579\\
8.5	1.32553342261403\\
9.5	1.13722623307345\\
10.5	0.982710353836723\\
11.5	0.819693292513591\\
14	0.711417512362161\\
15	0.594845343099834\\
16	0.432374779016712\\
18.5	0.35853080610066\\
20	0.277946474106088\\
21	0.247235141528847\\
22	0.219499863280106\\
23.5	0.19937754605192\\
24.5	0.170518176208034\\
26	0.154015097912662\\
27	0.142799712703138\\
28	0.132040754078204\\
30	0.126635284576622\\
31.5	0.12305680655237\\
32.5	0.120846350415462\\
34	0.119679270496448\\
35	0.118241557446865\\
36.5	0.117489692331716\\
37.5	0.116797908882376\\
38.5	0.115813168626033\\
40.5	0.115034546894748\\
42	0.114543085446499\\
43	0.114264496977512\\
44.5	0.11411543383123\\
45.5	0.114004778435444\\
47	0.113930846610643\\
48	0.113889627848508\\
49.5	0.113865722104023\\
50.5	0.113845672257568\\
52	0.113836047086664\\
53	0.11383043739329\\
54.5	0.113826133354749\\
55.5	0.113823198744967\\
};
\addlegendentry{CG};

\end{axis}
\end{tikzpicture}%
}
  }
\caption{\small
\label{fig:iter2}\small Comparison of optimization methods on year predict data ($d=90$, $n= 515345$). Y-axis:
best cost minus current cost values. X-axis: number of function and gradient evaluations. Right: 3 number of components. Left: 7 number of components.}
\end{figure}

In the next experiments, we compare the performance of manifold optimization methods on the reformulated problem and EM on some real datasets. One of the datasets is a dataset of natural images~\cite{hosseini2015matrix}. The other three datasets called `corel', `yearpredict' and `wine' data are taken from UCI machine learning dataset repository\footnote{Available via https://archive.ics.uci.edu/ml/datasets}. The results are shown in Figure~\ref{fig:iter1}-\ref{fig:iter4}. The dimensionality $d$ of data and number of data-points $n$ are given in the figure legends.

It can be seen than deterministic manifold optimization methods achieve and outperforms the EM algorithm. The manifold SGD shows remarkable performance. This method leads to fast increase of the objective function in early iterations. 

\begin{figure}[htbp]
\centering
\subfigure{%
  \label{fig:magic3}%
  \resizebox{.45\textwidth}{!}{
%
%
\definecolor{mycolor1}{rgb}{0.00000,0.00000,0.17241}%
\definecolor{mycolor2}{rgb}{1.00000,0.10345,0.72414}%
\definecolor{mycolor3}{rgb}{1.00000,0.82759,0.00000}%
\begin{tikzpicture}

\begin{axis}[%
scale = 0.55,
width=6.028in,
height=4.754in,
at={(1.011in,0.642in)},
scale only axis,
xmin=0,
xmax=120,
tick align=outside,
xlabel={Iterations},
ymode=log,
ymin=1e-05,
ymax=100,
ylabel={Averaged Cost Difference},
axis background/.style={fill=white},
legend style={legend cell align=left,align=left,draw=white!15!black}
]
\addplot [color=blue,solid,line width=3.0pt]
  table[row sep=crcr]{%
0	17.3888028265619\\
1	1.44085611502505\\
2	0.377997523072636\\
3	0.0625432709427693\\
4	0.00677466491306511\\
};
\addlegendentry{SGD (it=5)};

\addplot [color=red,solid,line width=3.0pt]
  table[row sep=crcr]{%
0	17.3888028265619\\
1	6.09967953824393\\
2	3.13537959954937\\
3	2.53026068313339\\
4	1.45060840016362\\
5	0.814299396460924\\
6	0.794877249642169\\
7	0.494519618344618\\
8	0.363431852393036\\
9	0.194043322031733\\
10	0.150236846937343\\
11	0.0682752276044294\\
12	0.043429685257018\\
13	0.0184042768824737\\
14	0.00788729782373299\\
15	0.00358697289303134\\
16	0.00198531376384548\\
17	0.000873792921059824\\
18	0.000477471021011588\\
19	0.000239658237984663\\
};
\addlegendentry{SGD (it=20)};

\addplot [color=green,solid,line width=3.0pt]
  table[row sep=crcr]{%
0	17.3888028265619\\
1	5.270525263687\\
2	6.75251507315913\\
3	3.88309356881229\\
4	3.15237462685813\\
5	3.34867009716051\\
6	2.6189533317209\\
7	2.25509477941612\\
8	2.04426526935734\\
9	1.7048608695777\\
10	1.8296455757707\\
11	1.21326134054808\\
12	1.01954023068535\\
13	0.958674692603101\\
14	0.808903288754087\\
15	0.762224712301301\\
16	0.615207911087317\\
17	0.452040028802383\\
18	0.424609426187994\\
19	0.354423899472213\\
20	0.335166987152664\\
21	0.306236488906026\\
22	0.214049109180226\\
23	0.194975836889424\\
24	0.16469291316688\\
25	0.11956874546049\\
26	0.108602081257995\\
27	0.0846684832777482\\
28	0.0603535157455086\\
29	0.0514402970172165\\
30	0.0401199890404129\\
31	0.0294499669744717\\
32	0.023995648641459\\
33	0.0157821436858674\\
34	0.0108072718125127\\
35	0.00793209488492153\\
36	0.00609511586725198\\
37	0.00404067174768219\\
38	0.00252798164548729\\
39	0.00185297334713308\\
40	0.00138959440731679\\
41	0.000928376987582524\\
42	0.000633074236875864\\
43	0.000412279396876158\\
44	0.000256978552637399\\
45	0.000154816621667386\\
46	9.83617992327268e-05\\
47	6.60102987843914e-05\\
48	2.49887666665671e-05\\
49	0\\
};
\addlegendentry{SGD (it=50)};

\addplot [color=mycolor1,solid,line width=3.0pt]
  table[row sep=crcr]{%
0	17.3888026061033\\
1	8.3559149076287\\
2	5.48375663358349\\
3	3.62919266694234\\
4	2.52058931791813\\
5	1.86430513697013\\
6	1.44362982486132\\
7	1.17927626174011\\
8	1.00199507216356\\
9	0.89642786798647\\
10	0.81287614316279\\
11	0.749044745742731\\
12	0.685458761455526\\
13	0.621196027334427\\
14	0.56646263002318\\
15	0.522881967265217\\
16	0.461724213655734\\
17	0.382579567855359\\
18	0.343475863542118\\
19	0.321528615653158\\
20	0.30489314745876\\
21	0.290683352820144\\
22	0.277089828728124\\
23	0.264849359383232\\
24	0.253008231686186\\
25	0.242169916786884\\
26	0.23185139891927\\
27	0.22064013411869\\
28	0.2085959027688\\
29	0.198184039732007\\
30	0.187699226167486\\
31	0.175895329186856\\
32	0.165215953080161\\
33	0.157517801736036\\
34	0.151324502947009\\
35	0.145544456526415\\
36	0.140619791746493\\
37	0.136237229969004\\
38	0.132693438183295\\
39	0.129680727761798\\
40	0.126828931114229\\
41	0.123888360990867\\
42	0.121162615259456\\
43	0.119014502209298\\
44	0.117214151914673\\
45	0.115980672308683\\
46	0.115043089179551\\
47	0.114297255243144\\
48	0.113547295581074\\
49	0.112726518125061\\
50	0.112010593759186\\
51	0.111485324972344\\
52	0.111135705752107\\
53	0.110947822841792\\
54	0.110833439103867\\
55	0.110753468293685\\
56	0.110691972079662\\
57	0.110640556682704\\
58	0.110597556162515\\
59	0.11056097285581\\
60	0.110527722020566\\
61	0.110495857719727\\
62	0.110464128171936\\
63	0.110431793776351\\
64	0.11039892531655\\
65	0.110366207216192\\
66	0.110334562965203\\
67	0.110304877362169\\
68	0.110277283301784\\
69	0.11025158292873\\
70	0.110227561597508\\
71	0.110204933010219\\
72	0.110183412162632\\
73	0.110162930822984\\
74	0.110143663123693\\
75	0.110125777524038\\
76	0.110109288557458\\
77	0.110094091954986\\
78	0.110079971363207\\
79	0.110066535666872\\
80	0.110053214360243\\
81	0.110039494655222\\
82	0.11002521387538\\
83	0.110009999140289\\
84	0.109992608070318\\
85	0.109973718858512\\
86	0.109957841016701\\
87	0.109945870203127\\
88	0.109936415263081\\
89	0.109928680460359\\
90	0.10992221592878\\
91	0.109916660426251\\
92	0.109911737782017\\
93	0.109907286212597\\
94	0.109903214432252\\
95	0.10989947900469\\
96	0.109896112613728\\
97	0.109893161681717\\
98	0.109890524239535\\
99	0.109887884469144\\
100	0.109884751802462\\
101	0.109880579007424\\
102	0.10987527167336\\
103	0.109869878148176\\
104	0.109865738455322\\
105	0.10986305410813\\
106	0.109861365805958\\
107	0.109860279536912\\
};
\addlegendentry{EM};

\addplot [color=mycolor2,solid,line width=3.0pt]
  table[row sep=crcr]{%
0	17.3888028265619\\
1	16.1494556975174\\
2	13.6020084203079\\
3	8.83288875914324\\
5	4.87808911686159\\
7	3.12170780930039\\
9	2.03515617463904\\
11	1.49677664420756\\
13	1.266356041244\\
15	1.10459953650474\\
16	0.919717003241444\\
17	0.78252774767487\\
18	0.619298880999029\\
19	0.48043356548587\\
20	0.348983796445506\\
21	0.297001062185684\\
22	0.226363666667667\\
23	0.174663426132655\\
24	0.144262623295624\\
26	0.12848414953407\\
28	0.120767581047207\\
29	0.119516636990078\\
30	0.117164208424875\\
31	0.113897234030918\\
33	0.112121612591115\\
34	0.111289323049467\\
35	0.110811818512803\\
36	0.110655128807762\\
37	0.110408401397754\\
38	0.110137094605378\\
40	0.109988361381888\\
41	0.10995617120523\\
42	0.109919459812882\\
43	0.10991312870915\\
44	0.109901004077427\\
45	0.109881924856704\\
47	0.109871904564404\\
48	0.109865313198\\
49	0.109862544199825\\
51	0.109861010034766\\
};
\addlegendentry{LBFGS};

\addplot [color=mycolor3,solid,line width=3.0pt]
  table[row sep=crcr]{%
0	17.3888028265619\\
1	16.1494556975174\\
2	14.212859590065\\
3	11.391645498816\\
4	7.65736335734263\\
5	4.33555073704268\\
7	2.92184928250871\\
8.5	2.06197598096909\\
9.5	1.65832030946167\\
10.5	1.30552156215966\\
12	1.14253335428832\\
13	0.954362342302955\\
15	0.802718149392022\\
16.5	0.612153709238587\\
17.5	0.503196689805097\\
18.5	0.405129325414677\\
21	0.346778672625537\\
22	0.272586022257596\\
23.5	0.239731633150587\\
24.5	0.220282988825462\\
25.5	0.193500890880282\\
27	0.157770703620189\\
28	0.136566563032718\\
30	0.127705031639945\\
31.5	0.120683286982429\\
32.5	0.117797367550384\\
33.5	0.114459832981787\\
35	0.112479094095282\\
37	0.111684583019207\\
38	0.111175438392211\\
39.5	0.110805486553225\\
40.5	0.110588105926656\\
42	0.11046374357963\\
43	0.110371898279197\\
44	0.110261724941933\\
46	0.110110941963901\\
47	0.109982943858868\\
49	0.109887700747394\\
50.5	0.109814875182012\\
51.5	0.109795355551511\\
52.5	0.109774093271535\\
54	0.109764184438663\\
55.5	0.109752456848729\\
56.5	0.109750176250884\\
57.5	0.109745511066999\\
59.5	0.109725150596034\\
62	0.109718047772339\\
63	0.109699888213033\\
64	0.109686495828727\\
65	0.109666420402851\\
66	0.109647492009624\\
68.5	0.109639601429959\\
69.5	0.109634440964233\\
71.5	0.109632887847083\\
};
\addlegendentry{CG};

\end{axis}
\end{tikzpicture}%
}
  }%
\hfill%
\subfigure{%
  \label{fig:year3}%
  \resizebox{.45\textwidth}{!}{\input{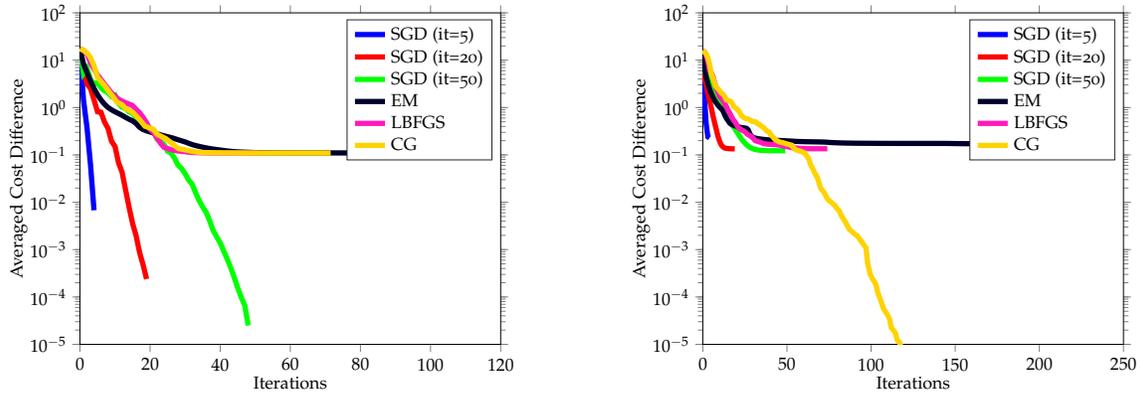}}
  }
\caption{\small
\label{fig:iter3}\small Comparison of optimization methods on corel data ($d=57$, $n=68040$). Y-axis:
current objective values minus best objective. X-axis: number of function and gradient evaluations. Right: 3 number of components. Left: 7 number of components.}
\end{figure}

\begin{figure}[htbp]
\centering
\subfigure{%
  \label{fig:magic4}%
  \resizebox{.45\textwidth}{!}{
%
%
\definecolor{mycolor1}{rgb}{0.00000,0.00000,0.17241}%
\definecolor{mycolor2}{rgb}{1.00000,0.10345,0.72414}%
\definecolor{mycolor3}{rgb}{1.00000,0.82759,0.00000}%
\begin{tikzpicture}

\begin{axis}[%
scale =0.55,
width=6.028in,
height=4.754in,
at={(1.011in,0.642in)},
scale only axis,
xmin=0,
xmax=100,
tick align=outside,
xlabel={Iterations},
ymode=log,
ymin=1e-05,
ymax=10,
ylabel={Averaged Cost Difference},
axis background/.style={fill=white},
legend style={at={(0.972,0.628)},legend cell align=left,align=left,draw=white!15!black}
]
\addplot [color=blue,solid,line width=3.0pt]
  table[row sep=crcr]{%
0	2.00669649394824\\
1	0.820696836227334\\
2	0.437253407359825\\
3	0.379731626890883\\
4	0.377846353575461\\
};
\addlegendentry{SGD (it=5)};

\addplot [color=red,solid,line width=3.0pt]
  table[row sep=crcr]{%
0	2.00669649394824\\
1	0.295708466204936\\
2	0.302265547936472\\
3	0.210562249910631\\
4	0.112067054277569\\
5	0.0842704721121614\\
6	0.0491348865880461\\
7	0.0226307480131815\\
8	0.015455436452406\\
9	0.00922690023059891\\
10	0.00401520911781761\\
11	0.00209997927367533\\
12	0.000566181836049839\\
13	0.000272698647131975\\
14	0.000113868486722346\\
15	6.1809562008186e-05\\
16	4.38729892011658e-05\\
17	3.7808082279156e-05\\
18	3.34259978074236e-05\\
19	3.11380900677172e-05\\
};
\addlegendentry{SGD (it=20)};

\addplot [color=green,solid,line width=3.0pt]
  table[row sep=crcr]{%
0	2.00669649394824\\
1	0.246412494386738\\
2	0.221127942140537\\
3	0.15750652309676\\
4	0.135441819128121\\
5	0.171607223968437\\
6	0.109129132282117\\
7	0.0833887968611493\\
8	0.0764051752087016\\
9	0.0568038039311927\\
10	0.0582495552200273\\
11	0.0376466082496156\\
12	0.0305373872622976\\
13	0.0306560113517409\\
14	0.0251461529917325\\
15	0.019616024014014\\
16	0.0199921742308757\\
17	0.0129966460166164\\
18	0.0105182192904176\\
19	0.00869856249044743\\
20	0.00622523439397371\\
21	0.00504846794652281\\
22	0.00367202533756039\\
23	0.00306211893034281\\
24	0.00244003077230293\\
25	0.00201145386806667\\
26	0.0016576206404304\\
27	0.00165429399627248\\
28	0.00152112160879669\\
29	0.00146558163509392\\
30	0.00135262642469902\\
31	0.00134108134779609\\
32	0.00130994102445081\\
33	0.00128630172796873\\
34	0.0012736812761327\\
35	0.00126771944456916\\
36	0.00126206311808152\\
37	0.00125794493158793\\
38	0.00125581025064925\\
39	0.00125396777934883\\
40	0.00125298362591364\\
};
\addlegendentry{SGD (it=50)};

\addplot [color=mycolor1,solid,line width=3.0pt]
  table[row sep=crcr]{%
0	2.00669418519404\\
1	0.959382390721523\\
2	0.752466897047209\\
3	0.654542889152454\\
4	0.589563883591373\\
5	0.537815741343934\\
6	0.509564246635912\\
7	0.492995975870769\\
8	0.48008275028456\\
9	0.461780578899031\\
10	0.450477802735478\\
11	0.444846747699845\\
12	0.440798597572025\\
13	0.437728228438425\\
14	0.43553298638902\\
15	0.433613954275626\\
16	0.431999141296714\\
17	0.430590230708003\\
18	0.429357634185949\\
19	0.428267599703471\\
20	0.427283866484979\\
21	0.42637895406832\\
22	0.425531114075595\\
23	0.424725322772072\\
24	0.423958654857868\\
25	0.423239380710672\\
26	0.422573500921629\\
27	0.42195084188345\\
28	0.421352654393896\\
29	0.420790333810002\\
30	0.420274985510172\\
31	0.419788001479791\\
32	0.419312160878418\\
33	0.418832522400085\\
34	0.418331573793018\\
35	0.417799452569692\\
36	0.417282273163328\\
37	0.416829571871758\\
38	0.416437505140193\\
39	0.41609798773327\\
40	0.4158054211544\\
41	0.415554239319607\\
42	0.415338696640638\\
43	0.415153259116485\\
44	0.414992966009537\\
45	0.414853557075185\\
46	0.414731390334682\\
47	0.414623295846499\\
48	0.414526473942358\\
49	0.414438446539297\\
50	0.414357024945521\\
51	0.414280269844051\\
52	0.414206441492388\\
53	0.414133949140242\\
54	0.414061311294948\\
55	0.413987139368864\\
56	0.41391015857409\\
57	0.413829279069112\\
58	0.413743721284139\\
59	0.413653177162885\\
60	0.413557960533997\\
61	0.413459090152206\\
62	0.413358278957192\\
63	0.413257841268841\\
64	0.413160507368357\\
65	0.413069076627995\\
66	0.412985901799638\\
67	0.41291239846091\\
68	0.412848841815947\\
69	0.412794504071271\\
70	0.4127479599402\\
71	0.412707369282491\\
72	0.412670654280228\\
73	0.412635600870353\\
74	0.41260003350859\\
75	0.412562336939187\\
76	0.412522385792699\\
77	0.412481856126939\\
78	0.412442514399659\\
79	0.412403756663201\\
80	0.412360222872663\\
81	0.412289623609545\\
82	0.412124202465786\\
83	0.411957628977619\\
84	0.411904114573223\\
85	0.411878170002604\\
86	0.411862080835409\\
87	0.41185137145125\\
88	0.411844067029093\\
89	0.411839033082842\\
90	0.411835546874951\\
91	0.41183312766975\\
92	0.411831448216761\\
93	0.411830282805382\\
};
\addlegendentry{EM};

\addplot [color=mycolor2,solid,line width=3.0pt]
  table[row sep=crcr]{%
0	2.00669649394824\\
1	1.61543703072116\\
2	0.933983109316276\\
4	0.635604630935704\\
6	0.556358262493474\\
7	0.508507974003031\\
8	0.44627272556699\\
9	0.427199092359692\\
10	0.399150573784033\\
11	0.345229043757005\\
12	0.274008434206272\\
14	0.249349189723165\\
15	0.230593063386524\\
16	0.221998071671166\\
17	0.215737668220308\\
19	0.21260156232884\\
20	0.21063313414446\\
21	0.20891359543875\\
23	0.208028429970509\\
25	0.207580660760898\\
27	0.207348240873208\\
28	0.207268308158759\\
29	0.207174808261996\\
30	0.207035862438086\\
31	0.206739586466243\\
34	0.164917429854481\\
38	0.164619477140919\\
41	0.123002308796939\\
42	0.0553162069831363\\
44	0.0359318405514477\\
46	0.0250787969758259\\
47	0.0146039029125173\\
49	0.0104985086088596\\
50	0.00860990840497111\\
51	0.00688240545931773\\
52	0.00525452667053727\\
53	0.00433835200186694\\
54	0.0034507446651264\\
55	0.00235852753670374\\
56	0.00175679148502716\\
57	0.00154671903284598\\
58	0.00117118666255278\\
59	0.00104798481785506\\
60	0.000880734426214502\\
61	0.000641497393140789\\
62	0.00055627020372695\\
63	0.000390142634092427\\
64	0.000243762824037486\\
65	0.000173067241513003\\
66	8.86949460516817e-05\\
68	4.91023498541132e-05\\
69	3.58324953526612e-05\\
70	1.93343235697085e-05\\
71	1.22085183567044e-05\\
72	7.07796985954801e-06\\
73	3.67228952935506e-06\\
75	1.75154904358266e-06\\
76	3.7078192161033e-07\\
};
\addlegendentry{LBFGS};

\addplot [color=mycolor3,solid,line width=3.0pt]
  table[row sep=crcr]{%
0	2.00669649394824\\
1	1.61543703072116\\
2	1.12635079888182\\
3	0.684282052668788\\
5	0.579575742004967\\
6.5	0.534151540783516\\
7.5	0.50991399573691\\
8.5	0.473165200584535\\
9.5	0.410047413944654\\
11.5	0.338062840919282\\
13	0.294148163984668\\
14	0.260270257313393\\
16	0.2440903433258\\
17	0.230031237154915\\
18.5	0.219442758773595\\
19.5	0.211014664297353\\
21.5	0.209212506476474\\
22.5	0.208587647670667\\
24	0.208231942227053\\
25	0.207756689816412\\
26	0.207613902845711\\
28	0.207322434456963\\
30	0.207274504874336\\
31	0.207232139841818\\
32.5	0.207193731495018\\
33.5	0.207149265636251\\
35	0.207131213592307\\
36	0.207114597824719\\
37.5	0.207102407021187\\
38.5	0.207093149164791\\
39.5	0.207082760604216\\
40.5	0.206886747091499\\
41.5	0.195389046655449\\
44.5	0.130822475638414\\
46.5	0.0573419995578131\\
48	0.0373303773925628\\
49	0.0242451223632809\\
50.5	0.0197437066410475\\
51.5	0.0154414676082961\\
52.5	0.0110088896312321\\
54	0.00319363663953709\\
55.5	0.00203846498015148\\
56.5	0.00148886826179595\\
57.5	0.000949834421966944\\
59	0.000536864833201367\\
60.5	0.000238670853088152\\
61.5	0.000171831618832829\\
62.5	8.91725412066613e-05\\
64	4.24310339126599e-05\\
65.5	1.48048873631623e-05\\
66.5	1.03614144730813e-05\\
67.5	5.84608778098072e-06\\
69	1.84427489191563e-06\\
70.5	0\\
};
\addlegendentry{CG};

\end{axis}
\end{tikzpicture}%
}
  }%
\hfill%
\subfigure{%
  \label{fig:year4}%
  \resizebox{.45\textwidth}{!}{
%
%
\definecolor{mycolor1}{rgb}{0.00000,0.00000,0.17241}%
\definecolor{mycolor2}{rgb}{1.00000,0.10345,0.72414}%
\definecolor{mycolor3}{rgb}{1.00000,0.82759,0.00000}%
\begin{tikzpicture}

\begin{axis}[%
scale =0.55,
width=6.028in,
height=4.754in,
at={(1.011in,0.642in)},
scale only axis,
xmin=0,
xmax=120,
tick align=outside,
xlabel={Iterations},
ymode=log,
ymin=1e-05,
ymax=10,
ylabel={Averaged Cost Difference},
axis background/.style={fill=white},
legend style={legend cell align=left,align=left,draw=white!15!black},
legend pos=south east
]
\addplot [color=blue,solid,line width=3.0pt]
  table[row sep=crcr]{%
0	2.64817690864674\\
1	0.310326266404217\\
2	0.104066445257155\\
3	0.0673479079224595\\
4	0.064692860296109\\
};
\addlegendentry{SGD (it=5)};

\addplot [color=red,solid,line width=3.0pt]
  table[row sep=crcr]{%
0	2.64817690864674\\
1	0.487536180209704\\
2	0.332634249856547\\
3	0.172316213376184\\
4	0.0909098106142192\\
5	0.0480784297221162\\
6	0.03354061866673\\
7	0.0179043623634632\\
8	0.0116289688597857\\
9	0.00860395941081982\\
10	0.00686905674017257\\
11	0.00624834777159933\\
12	0.00594249575431882\\
13	0.00579087115727028\\
14	0.00570368394867216\\
15	0.00564273061388221\\
16	0.0056000240573959\\
17	0.00557364258415438\\
18	0.00555283204758394\\
19	0.00553921944525415\\
};
\addlegendentry{SGD (it=20)};

\addplot [color=green,solid,line width=3.0pt]
  table[row sep=crcr]{%
0	2.64817690864674\\
1	0.441641871912891\\
2	0.364483884538868\\
3	0.247688495510263\\
4	0.235852167958413\\
5	0.190174289647566\\
6	0.176505185153844\\
7	0.151274813571641\\
8	0.123460058657933\\
9	0.0743848772880942\\
10	0.0626250184511761\\
11	0.0411286179569053\\
12	0.0428106860631978\\
13	0.0344416999677808\\
14	0.0265693053537213\\
15	0.0187012491595382\\
16	0.013475594028606\\
17	0.01202017671314\\
18	0.00955143851553641\\
19	0.00670415670603286\\
20	0.00527904685812142\\
21	0.00331535489016277\\
22	0.00256985303048607\\
23	0.00205495726311189\\
24	0.00136368641991691\\
25	0.000757244481244479\\
26	0.000647516584016072\\
27	0.000378975481512001\\
28	0.000291007031226531\\
29	0.000230225985636379\\
30	0.000151317701902354\\
31	0.000101703035129574\\
32	7.51456275576157e-05\\
33	5.74442655409957e-05\\
34	4.60651473936924e-05\\
35	3.62097060417899e-05\\
36	3.07937279087067e-05\\
37	2.33584296847056e-05\\
38	1.79322022564321e-05\\
39	1.4333421551882e-05\\
40	1.10410357891944e-05\\
41	8.7474254126807e-06\\
42	6.44762868851068e-06\\
43	4.7369737692371e-06\\
44	3.04018227481606e-06\\
45	2.02022430806359e-06\\
46	8.68652156871264e-07\\
47	0\\
};
\addlegendentry{SGD (it=50)};

\addplot [color=mycolor1,solid,line width=3.0pt]
  table[row sep=crcr]{%
0	2.64817152155779\\
1	1.28030886467139\\
2	0.90784507342694\\
3	0.666015889807503\\
4	0.535113811754011\\
5	0.449913144542674\\
6	0.39118717370149\\
7	0.347572662157617\\
8	0.314999732799629\\
9	0.288460169721983\\
10	0.265972042806076\\
11	0.249662606361756\\
12	0.235745769612442\\
13	0.227453498484093\\
14	0.221686320704929\\
15	0.216628008779775\\
16	0.21196078718381\\
17	0.208097145525092\\
18	0.205089745070594\\
19	0.202519965409536\\
20	0.200272272593271\\
21	0.198077904363339\\
22	0.195936408681319\\
23	0.194022511741536\\
24	0.192430480857629\\
25	0.191130765810231\\
26	0.190034817107906\\
27	0.189058343314368\\
28	0.188149814714589\\
29	0.187286294653373\\
30	0.186472117293908\\
31	0.185742751364905\\
32	0.185141139232623\\
33	0.184680679516898\\
34	0.184340211006653\\
35	0.184081069952243\\
36	0.183864482911587\\
37	0.183660712785095\\
38	0.183448346731876\\
39	0.183209943869467\\
40	0.182928882070757\\
41	0.182588492369965\\
42	0.182172306448394\\
43	0.181662316267692\\
44	0.181032492762694\\
45	0.180239583109544\\
46	0.179222599174328\\
47	0.177935800054381\\
48	0.176407184029104\\
49	0.174701210464524\\
50	0.17281563406786\\
51	0.170786131668403\\
52	0.168727291981542\\
53	0.166605179626648\\
54	0.164317736424356\\
55	0.161568947066112\\
56	0.157844225387571\\
57	0.152670838828482\\
58	0.146390253070912\\
59	0.139702152894248\\
60	0.132670285215998\\
61	0.12557161871228\\
62	0.118313586225701\\
63	0.10962087226309\\
64	0.0982885041595876\\
65	0.0845842412478766\\
66	0.0708602868503283\\
67	0.0601240777670558\\
68	0.052571577276088\\
69	0.047181139573045\\
70	0.0429812640547835\\
71	0.04063119351429\\
72	0.0395338319924778\\
73	0.0388806078690329\\
74	0.0384541083239918\\
75	0.0381666765169926\\
76	0.0379693852164948\\
77	0.0378317934373347\\
78	0.0377340591988646\\
79	0.0376628862069419\\
80	0.0376091926394158\\
81	0.0375666728550026\\
82	0.0375308456730896\\
83	0.0374983400196223\\
84	0.0374662263538794\\
85	0.0374313244054552\\
86	0.0373900566522627\\
87	0.0373411742476244\\
88	0.0372910722774331\\
89	0.0372454582958546\\
90	0.0371920392259291\\
91	0.0371310938631484\\
92	0.0370995577538409\\
93	0.0370890174614695\\
94	0.0370841661501937\\
95	0.0370813362608318\\
96	0.037079516776249\\
97	0.0370782898613395\\
};
\addlegendentry{EM};

\addplot [color=mycolor2,solid,line width=3.0pt]
  table[row sep=crcr]{%
0	2.64817690864674\\
1	2.33386572226284\\
2	1.70008598713725\\
3	0.960610685214921\\
5	0.646233557247752\\
7	0.509568450681735\\
8	0.465364173870377\\
9	0.408054481489201\\
10	0.323418246303089\\
12	0.283880419244706\\
14	0.257190882308667\\
15	0.235463554173266\\
16	0.227529166916913\\
17	0.214952251382508\\
18	0.201360974990301\\
19	0.193987454768066\\
20	0.189337409726376\\
21	0.184733434199259\\
22	0.180625285627013\\
23	0.176948736769222\\
24	0.174582602119947\\
25	0.1721027740565\\
26	0.170150725691446\\
27	0.167431824926854\\
28	0.163173709521942\\
30	0.161279922932258\\
31	0.159744176805512\\
33	0.158858389072634\\
34	0.158555632573046\\
35	0.15820210725795\\
36	0.158156062497204\\
37	0.158066450858899\\
38	0.157908547120342\\
39	0.15772718209762\\
40	0.157620300127616\\
41	0.157586052166945\\
42	0.15754396431892\\
43	0.15751499846172\\
44	0.157495918388463\\
45	0.157482497647635\\
46	0.157476553341227\\
47	0.157472870480812\\
48	0.15746889872626\\
49	0.157465474388572\\
51	0.157463250178269\\
};
\addlegendentry{LBFGS};

\addplot [color=mycolor3,solid,line width=3.0pt]
  table[row sep=crcr]{%
0	2.64817690864674\\
1	2.33386572226284\\
2	1.83586033311682\\
3	1.14276999024222\\
4.5	0.822493587791725\\
5.5	0.619376292586459\\
6.5	0.489941750538123\\
9	0.439095684703062\\
10	0.403741979523174\\
11	0.356703482927521\\
13	0.319991130200628\\
14.5	0.293002760238904\\
15.5	0.278572483295926\\
16.5	0.252930674222831\\
18.5	0.238305324481125\\
19.5	0.221140708270633\\
21.5	0.212452825377849\\
22.5	0.203639932094823\\
24.5	0.20010424734591\\
25.5	0.196686680571532\\
27.5	0.194920523464931\\
28.5	0.192870151330619\\
30	0.191139634785098\\
31.5	0.190278733567885\\
32.5	0.189783368687533\\
33.5	0.18933053054303\\
35	0.188911306368895\\
36	0.188373362131295\\
38	0.188175772668875\\
39	0.187888013837713\\
40	0.187306153901022\\
42	0.186497035069889\\
44.5	0.18548990330022\\
46	0.183365531311683\\
47	0.179312068654624\\
48	0.170516451891294\\
49	0.155430176135879\\
52.5	0.146293311450843\\
53.5	0.140514783599481\\
54.5	0.132988598404824\\
55.5	0.122239331523398\\
56.5	0.0834919290540119\\
59	0.071391653582437\\
60	0.0477087863863035\\
62	0.0445828586240824\\
63	0.0431623492371394\\
64	0.0419716708752362\\
66	0.0410574386064266\\
67.5	0.0403139739478562\\
68.5	0.0398734335955773\\
69.5	0.039472318037213\\
70.5	0.038928707696799\\
73.5	0.0385715386312149\\
75	0.0384055592926065\\
76	0.038336275988762\\
77	0.0382279277188313\\
78	0.0380666329301684\\
79.5	0.0379425071099653\\
80.5	0.0375603687374559\\
82	0.0374356022797504\\
83	0.0372677843932232\\
85.5	0.0372330029387582\\
86.5	0.0371976938661729\\
87.5	0.0371669710267089\\
89.5	0.0371333054483491\\
91.5	0.0371282691594028\\
92.5	0.0371249701688683\\
94	0.0371221777393684\\
95	0.0371192822197368\\
96	0.037116485803665\\
98	0.0371149963600792\\
99	0.0371124270677929\\
101	0.037111366048461\\
};
\addlegendentry{CG};

\end{axis}
\end{tikzpicture}%
}
  }
\caption{\small
\label{fig:iter4}\small Comparison of optimization methods on wine data ($d=11$, $n= 6497$). Y-axis:
current objective values minus best objective. X-axis: number of function and gradient evaluations. Right: 3 number of components. Left: 7 number of components.}
\end{figure}
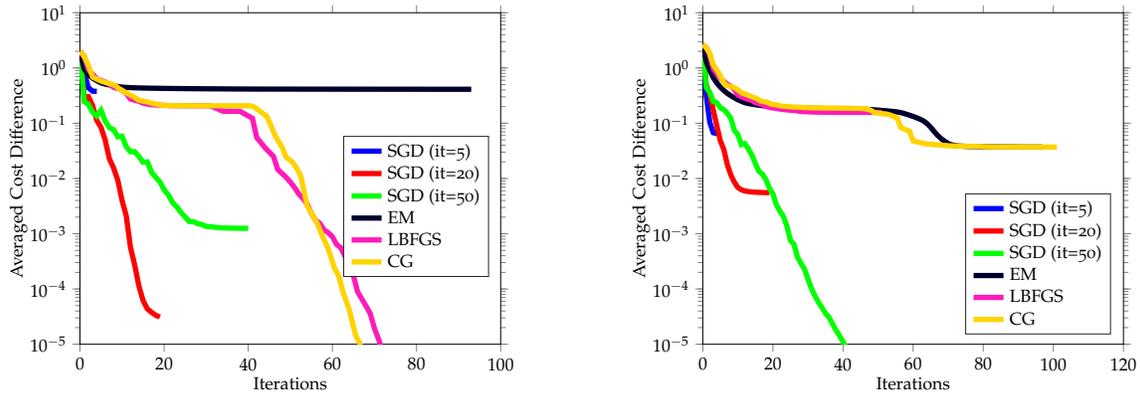

\section{Conclusions and future work}
In this paper, we proposed a reformulation for the \gmm problem that can make Riemannian manifold optimization a powerful alternative to the EM algorithm for fitting Gaussian mixture models.
The deterministic manifold optimization methods can either match or outperform EM algorithm.
Furthermore, we  developed a global convergence theory for SGD on manifolds. We applied this theory to the \gmm{} modeling. Experimentally Riemannian SGD for \gmm shows remarkable convergence behavior, making it a potential candidate for large scale mixture modeling.

There are several venues for future works, including 
extension of Riemannian optimization to estimation in hidden Markov models, an exploration of manifold optimization for non-Gaussian mixture models, and a study of richer priors for \gmm{}s beyond the usual conjugate priors.

\bibliographystyle{plainnat}
\setlength{\bibsep}{3pt}
\bibliography{riemmix}   

\begin{thebibliography}{43}
\providecommand{\natexlab}[1]{#1}
\providecommand{\url}[1]{\texttt{#1}}
\expandafter\ifx\csname urlstyle\endcsname\relax
  \providecommand{\doi}[1]{doi: #1}\else
  \providecommand{\doi}{doi: \begingroup \urlstyle{rm}\Url}\fi

\bibitem[Absil et~al.(2009)Absil, Mahony, and Sepulchre]{absil2009optimization}
P-A Absil, Robert Mahony, and Rodolphe Sepulchre.
\newblock \emph{Optimization algorithms on matrix manifolds}.
\newblock Princeton University Press, 2009.

\bibitem[Arthur and Vassilvitskii(2007)]{arthur2007}
David Arthur and Sergei Vassilvitskii.
\newblock k-means++: The advantages of careful seeding.
\newblock In \emph{18th Annual ACM-SIAM Symposium on Discrete Algorithms
  (SODA)}, pages 1027--1035, 2007.

\bibitem[Balakrishnan et~al.(2014)Balakrishnan, Wainwright, and
  Yu]{balakrishnan2014statistical}
Sivaraman Balakrishnan, Martin~J Wainwright, and Bin Yu.
\newblock Statistical guarantees for the {EM} algorithm: From population to
  sample-based analysis.
\newblock \emph{arXiv:1408.2156}, 2014.

\bibitem[Bhatia(2007)]{bhatia07}
R.~Bhatia.
\newblock \emph{Positive Definite Matrices}.
\newblock Princeton University Press, 2007.

\bibitem[Bhojanapalli et~al.(2016)Bhojanapalli, Kyrillidis, and
  Sanghavi]{bhoj16}
Srinadh Bhojanapalli, Anastasios Kyrillidis, and Sujay Sanghavi.
\newblock Dropping convexity for faster semi-definite optimization.
\newblock In \emph{29th Annual Conference on Learning Theory (COLT)}, pages
  530--582, 2016.

\bibitem[Bishop(2007)]{bishop}
C.~M. Bishop.
\newblock \emph{Pattern recognition and machine learning}.
\newblock Springer, 2007.

\bibitem[Bonnabel(2013)]{bonnabel2013}
Silvere Bonnabel.
\newblock Stochastic gradient descent on {Riemannian} manifolds.
\newblock \emph{IEEE Transactions on Automatic Control}, 58\penalty0
  (9):\penalty0 2217--2229, 2013.

\bibitem[Boumal et~al.(2014)Boumal, Mishra, Absil, and
  Sepulchre]{boumal2014manopt}
Nicolas Boumal, Bamdev Mishra, P-A Absil, and Rodolphe Sepulchre.
\newblock Manopt, a matlab toolbox for optimization on manifolds.
\newblock \emph{The Journal of Machine Learning Research}, 15\penalty0
  (1):\penalty0 1455--1459, 2014.

\bibitem[Boumal et~al.(2016)Boumal, Absil, and Cartis]{Boumal2016}
Nicolas Boumal, P.-A Absil, and Coralia Cartis.
\newblock {Global rates of convergence for nonconvex optimization on
  manifolds}.
\newblock \emph{arXiv:1605.08101v1}, 2016.

\bibitem[Burer et~al.(1999)Burer, Monteiro, and Zhang]{burer1999solving}
Sam Burer, Renato~DC Monteiro, and Yin Zhang.
\newblock Solving semidefinite programs via nonlinear programming. part {I}:
  Transformations and derivatives.
\newblock Technical Report TR99-17, Department of Computational and Applied
  Mathematics, Rice University, Houston TX, 1999.

\bibitem[Dasgupta(1999)]{dasgupta1999learning}
Sanjoy Dasgupta.
\newblock Learning mixtures of {Gaussians}.
\newblock In \emph{40th Annual IEEE Symposium on Foundations of Computer
  Science (FOCS)}, pages 634--644, 1999.

\bibitem[Dempster et~al.(1977)Dempster, Laird, and Rubin]{dempster77}
A.~P. Dempster, N.~M. Laird, and D.~B. Rubin.
\newblock Maximum likelihood from incomplete data via the {EM} algorithm.
\newblock \emph{Journal of the Royal Statistical Society, Series B},
  39:\penalty0 1--38, 1977.

\bibitem[Duda et~al.(2000)Duda, Hart, and Stork]{dudahart}
R.~O. Duda, P.~E. Hart, and D.~G. Stork.
\newblock \emph{Pattern Classification}.
\newblock John Wiley \& Sons, 2nd edition, 2000.

\bibitem[Friedman et~al.(2001)Friedman, Hastie, and
  Tibshirani]{friedman2001elements}
Jerome Friedman, Trevor Hastie, and Robert Tibshirani.
\newblock \emph{The elements of statistical learning}.
\newblock Springer, 2001.

\bibitem[Ge et~al.(2015)Ge, Huang, and Kakade]{kakade15}
Rong Ge, Qingqing Huang, and Sham~M. Kakade.
\newblock Learning mixtures of {Gaussians} in high dimensions.
\newblock \emph{arXiv:1503.00424}, 2015.

\bibitem[Ghadimi and Lan(2013)]{Ghadimi13}
Saeed Ghadimi and Guanghui Lan.
\newblock Stochastic first- and zeroth-order methods for nonconvex stochastic
  programming.
\newblock \emph{{SIAM} Journal on Optimization}, 23\penalty0 (4):\penalty0
  2341--2368, 2013.

\bibitem[Hosseini and Sra(2015)]{hosseini2015matrix}
Reshad Hosseini and Suvrit Sra.
\newblock Matrix manifold optimization for {Gaussian} mixtures.
\newblock In \emph{Advances in Neural Information Processing Systems 28
  (NIPS)}, pages 910--918, 2015.

\bibitem[Jeuris et~al.(2012)Jeuris, Vandebril, and
  Vandereycken]{jeuris2012survey}
Ben Jeuris, Raf Vandebril, and Bart Vandereycken.
\newblock A survey and comparison of contemporary algorithms for computing the
  matrix geometric mean.
\newblock \emph{Electronic Transactions on Numerical Analysis}, 39:\penalty0
  379--402, 2012.

\bibitem[Jordan and Jacobs(1994)]{jordan1994hierarchical}
Michael~I Jordan and Robert~A Jacobs.
\newblock Hierarchical mixtures of experts and the {EM} algorithm.
\newblock \emph{Neural Computation}, 6\penalty0 (2):\penalty0 181--214, 1994.

\bibitem[Journ{\'e}e et~al.(2010)Journ{\'e}e, Bach, Absil, and
  Sepulchre]{journee2010low}
Michel Journ{\'e}e, Francis Bach, P-A Absil, and Rodolphe Sepulchre.
\newblock Low-rank optimization on the cone of positive semidefinite matrices.
\newblock \emph{SIAM Journal on Optimization}, 20\penalty0 (5):\penalty0
  2327--2351, 2010.

\bibitem[Keener(2010)]{keener}
R.~W. Keener.
\newblock \emph{{Theoretical Statistics}}.
\newblock Springer Texts in Statistics. Springer, 2010.

\bibitem[Lee(2012)]{lee12}
John~M. Lee.
\newblock \emph{Introduction to Smooth Manifolds}.
\newblock Springer, 2012.

\bibitem[Ma et~al.(2000)Ma, Xu, and Jordan]{ma2000asymptotic}
Jinwen Ma, Lei Xu, and Michael~I Jordan.
\newblock Asymptotic convergence rate of the {EM} algorithm for {Gaussian}
  mixtures.
\newblock \emph{Neural Computation}, 12\penalty0 (12):\penalty0 2881--2907,
  2000.

\bibitem[McLachlan and Peel(2000)]{McLPee00}
G.~J. McLachlan and D.~Peel.
\newblock \emph{Finite mixture models}.
\newblock John Wiley and Sons, 2000.

\bibitem[Moitra and Valiant(2010)]{moitra2010}
Ankur Moitra and Gregory Valiant.
\newblock Settling the polynomial learnability of mixtures of {Gaussians}.
\newblock In \emph{51st Annual IEEE Symposium on Foundations of Computer
  Science (FOCS)}, pages 93--102, 2010.

\bibitem[Murphy(2012)]{murphy12}
Kevin~P. Murphy.
\newblock \emph{Machine Learning: A Probabilistic Perspective}.
\newblock MIT Press, 2012.

\bibitem[Naim and Gildea(2012)]{naim2012convergence}
Iftekhar Naim and Daniel Gildea.
\newblock Convergence of the {EM} algorithm for {Gaussian} mixtures with
  unbalanced mixing coefficients.
\newblock In \emph{29th International Conference on Machine Learning (ICML)},
  pages 1655--1662, 2012.

\bibitem[Nocedal and Wright(2006)]{nocedal2006numerical}
Jorge Nocedal and Stephen~J. Wright.
\newblock \emph{Numerical Optimization}.
\newblock Springer, 2006.

\bibitem[Redner and Walker(1984)]{redWal84}
R.~A. Redner and H.~F. Walker.
\newblock Mixture densities, maximum likelihood, and the {EM} algorithm.
\newblock \emph{Siam Review}, 26:\penalty0 195--239, 1984.

\bibitem[Reynolds et~al.(2000)Reynolds, Quatieri, and
  Dunn]{reynolds2000speaker}
Douglas~A Reynolds, Thomas~F Quatieri, and Robert~B Dunn.
\newblock Speaker verification using adapted {Gaussian} mixture models.
\newblock \emph{Digital Signal Processing}, 10\penalty0 (1-3):\penalty0 19--41,
  2000.

\bibitem[Ridolfi et~al.(1999)Ridolfi, Idier, and
  Mohammad-Djafari]{ridolfi2001penalized}
Andrea Ridolfi, J{\'e}r{\^o}me Idier, and Ali Mohammad-Djafari.
\newblock Penalized maximum likelihood estimation for univariate normal mixture
  distributions.
\newblock In \emph{Actes du 17{$^e$} Colloque {GRETSI}}, pages 259--262, 1999.

\bibitem[Ring and Wirth(2012)]{ring2012optimization}
Wolfgang Ring and Benedikt Wirth.
\newblock Optimization methods on {Riemannian} manifolds and their application
  to shape space.
\newblock \emph{SIAM Journal on Optimization}, 22\penalty0 (2):\penalty0
  596--627, 2012.

\bibitem[Salakhutdinov et~al.(2003)Salakhutdinov, Roweis, and
  Ghahramani]{salakhutdinov2003optimization}
Ruslan Salakhutdinov, Sam~T Roweis, and Zoubin Ghahramani.
\newblock Optimization with {EM} and expectation-conjugate-gradient.
\newblock In \emph{20th International Conference on Machine Learning (ICML)},
  pages 672--679, 2003.

\bibitem[Sra and Hosseini(2013)]{sra2013geometric}
Suvrit Sra and Reshad Hosseini.
\newblock Geometric optimisation on positive definite matrices for elliptically
  contoured distributions.
\newblock In \emph{Advances in Neural Information Processing Systems 26
  (NIPS)}, pages 2562--2570, 2013.

\bibitem[Sra and Hosseini(2015)]{sra15}
Suvrit Sra and Reshad Hosseini.
\newblock Conic geometric optimization on the manifold of positive definite
  matrices.
\newblock \emph{SIAM Journal on Optimization}, 25\penalty0 (1):\penalty0
  713--739, 2015.

\bibitem[Udri\c{s}te(1994)]{udriste}
Constantin Udri\c{s}te.
\newblock \emph{Convex functions and optimization methods on {R}iemannian
  manifolds}.
\newblock Kluwer Academic, 1994.

\bibitem[Vanderbei and Benson(2000)]{vanderbei2000formulating}
Robert~J Vanderbei and H~Yurttan Benson.
\newblock On formulating semidefinite programming problems as smooth convex
  nonlinear optimization problems.
\newblock Technical Report ORFE-99-01, Department of Operations Research and
  Financial Engineering, Princeton University, Princeton NJ, 2000.

\bibitem[Vandereycken(2013)]{vandereycken2013low}
Bart Vandereycken.
\newblock Low-rank matrix completion by {Riemannian} optimization.
\newblock \emph{SIAM Journal on Optimization}, 23\penalty0 (2):\penalty0
  1214--1236, 2013.

\bibitem[Wiesel(2012)]{wiesel12}
A.~Wiesel.
\newblock Geodesic convexity and covariance estimation.
\newblock \emph{{IEEE} Transactions on Signal Processing}, 60\penalty0
  (12):\penalty0 6182--89, 2012.

\bibitem[Wisdom et~al.(2016)Wisdom, Powers, Hershey, Le~Roux, and
  Atlas]{wisdom2016full}
Scott Wisdom, Thomas Powers, John Hershey, Jonathan Le~Roux, and Les Atlas.
\newblock Full-capacity unitary recurrent neural networks.
\newblock In \emph{Advances in Neural Information Processing Systems 29
  (NIPS)}, pages 4880--4888, 2016.

\bibitem[Xu and Jordan(1996)]{jordan96}
L.~Xu and M.~I. Jordan.
\newblock On convergence properties of the {EM} algorithm for {Gaussian}
  mixtures.
\newblock \emph{Neural Computation}, 8:\penalty0 129--151, 1996.

\bibitem[Zhang and Sra(2016)]{zhangSra16a}
Hongyi Zhang and Suvrit Sra.
\newblock First-order methods for geodesically convex optimization.
\newblock In \emph{29th Annual Conference on Learning Theory (COLT)}, pages
  1617--1638, 2016.

\bibitem[Zhang et~al.(2016)Zhang, Reddi, and Sra]{rsvrg}
Hongyi Zhang, Sashank Reddi, and Suvrit Sra.
\newblock {Riemannian SVRG}: Fast stochastic optimization on {Riemannian}
  manifolds.
\newblock In \emph{Advances in Neural Information Processing Systems 29
  (NIPS)}, pages 4592--4600, 2016.

\end{thebibliography}

\end{document}